\documentclass{article}
\usepackage{iclr2027_conference,times}

\usepackage{amsmath}
\usepackage{amssymb}
\usepackage{amsthm}
\usepackage{mathtools}
\usepackage{bm}

\usepackage{graphicx}
\usepackage{pgffor}
\usepackage{subcaption}
\usepackage{booktabs}
\usepackage{multirow}
\usepackage{array}
\usepackage{arydshln}
\usepackage{wrapfig}
\usepackage{needspace}

\usepackage{algorithm}
\usepackage{algorithmic}

\usepackage{xspace}
\usepackage{xcolor}
\usepackage{microtype}
\usepackage{hyperref}
\usepackage[capitalize,nameinlink]{cleveref}

\hypersetup{
    colorlinks=true,
    linkcolor=blue,
    citecolor=blue,
    urlcolor=blue
}

\newtheorem{theorem}{Theorem}
\newtheorem{proposition}{Proposition}

\theoremstyle{definition}

\theoremstyle{remark}

\graphicspath{{figures/}}

\let\iclrstandardbf\bf

\newcommand{\R}{\mathbb{R}}

\renewcommand{\bf}[1]{\textbf{#1}}

\newcommand{\I}{\mathbf{I}}
\newcommand{\D}{\mathbf{D}}
\newcommand{\X}{\mathbf{X}}
\newcommand{\Xco}{\X^{\mathrm{co}}}
\newcommand{\Xreg}{\X^{\mathrm{reg}}}

\newcommand{\Xdown}{\X_{\mathrm{down}}}

\newcommand{\Xcohad}{\Xco_\mathrm{had}}

\newcommand{\A}{\mathbf{A}}
\newcommand{\Y}{\mathbf{Y}}
\newcommand{\V}{\mathbf{V}}
\newcommand{\W}{\mathbf{W}}

\newcommand{\Wdown}{\W_{\mathrm{down}}}

\newcommand{\Pmat}{\mathbf{P}}
\newcommand{\Q}{\mathbf{Q}}

\newcommand{\Z}{\mathbf{Z}}
\newcommand{\Hmat}{\mathbf{H}}

\newcommand{\Rmat}{\mathbf{R}}
\newcommand{\Kmat}{\mathbf{K}}

\newcommand{\Pimat}{\mathbf{\Pi}}
\newcommand{\Lambdamat}{\mathbf{\Lambda}}

\newcommand{\Vtilde}{\widetilde{\V}}
\newcommand{\Ztilde}{\widetilde{\Z}}

\newcommand{\entry}[3]{(#1)_{#2,#3}}
\newcommand{\Aentry}[2]{\entry{\A}{#1}{#2}}
\newcommand{\Xentry}[2]{\entry{\X}{#1}{#2}}
\newcommand{\Xdownentry}[2]{\entry{\Xdown}{#1}{#2}}
\newcommand{\Xcoentry}[2]{\entry{\Xco}{#1}{#2}}
\newcommand{\Xregentry}[2]{\entry{\Xreg}{#1}{#2}}

\newcommand{\Xcohadentry}[2]{\entry{\Xcohad}{#1}{#2}}

\newcommand{\Zentry}[2]{\entry{\Z}{#1}{#2}}
\newcommand{\Ztildeentry}[2]{\entry{\Ztilde}{#1}{#2}}
\newcommand{\Hentry}[2]{\entry{\Hmat}{#1}{#2}}

\newcommand{\Wentry}[2]{\entry{\W}{#1}{#2}}

\newcommand{\Wdownentry}[2]{\entry{\Wdown}{#1}{#2}}

\newcommand{\diag}{\operatorname{diag}}

\newcommand{\quant}{\mathcal{Q}}

\newcommand{\obj}{\mathcal{J}}

\newcommand{\objagwc}{\obj_{\mathrm{AGWC}}}
\newcommand{\objco}{\obj_{\mathrm{co}}}
\newcommand{\objreg}{\obj_{\mathrm{reg}}}

\newcommand{\order}{\mathcal{O}}
\newcommand{\Cco}{\mathcal{C}_{\mathrm{co}}{}}
\newcommand{\Nco}{N_{\mathrm{co}}{}}

\newcommand{\proposal}{L2-SmoothRot\xspace}
\newcommand{\smoothquant}{SmoothQuant\xspace}
\newcommand{\omniquant}{OmniQuant\xspace}
\newcommand{\flatquant}{FlatQuant\xspace}
\newcommand{\spinquant}{SpinQuant\xspace}
\newcommand{\quarot}{QuaRot\xspace}

\newcommand{\smoothrot}{SmoothRot\xspace}
\newcommand{\quip}{QuIP\xspace}
\newcommand{\quipsharp}{QuIP\#\xspace}
\newcommand{\dfrot}{DFRot\xspace}

\let\bf\iclrstandardbf

\title{
    Understanding LLM Quantization through Activation-Guided Compensation and Orthogonal Residuals
}

\author{
    Yamato Narita\textsuperscript{1} and Issei Sato\textsuperscript{2}\\
    Department of Computer Science, The University of Tokyo \\
    \texttt{\{narita-yamato644\textsuperscript{1},sato\textsuperscript{2}\}@g.ecc.u-tokyo.ac.jp}
}

\iclrfinalcopy

\begin{document}

\maketitle
\lhead{Preprint. Under review.}

\begin{abstract}
\noindent
Post-training weight-activation quantization reduces the memory and inference costs of large language models, but aggressive W4A4 quantization remains difficult because activation outliers degrade effective quantization resolution.
Although weight optimization, channel-wise scaling, and orthogonal rotation mitigate this problem, the error components they address and their relationship remain unclear.
Using an exact decomposition of local weight-activation quantization error into an activation-guided weight compensation term and an orthogonal residual, we bound the residual using persistent channel-wise outlier and regular activation quantities.
This decomposition clarifies which error components can be addressed by weight compensation and which require transformation design.
We then use the residual bounds to derive practical guidelines for applying randomized Hadamard rotation, sign selection, and channel scaling.
In particular, the analysis explains how random signs suppress constructive interference among persistent outlier channels, how sampling multiple sign patterns can improve transformation selection, and how second-moment balancing leads to an $L_2$ scaling rule while a further relaxation recovers SmoothQuant-style $L_\infty$ scaling.
We evaluate these guidelines through backpropagation-free configurations across eight Llama and Mistral models, obtaining performance competitive with gradient-trained SpinQuant.
\end{abstract}

\section{Introduction}

Large language models (LLMs) achieve strong performance across a wide range of tasks, but their growing memory and inference costs hinder practical deployment.
Post-training quantization (PTQ) reduces these costs using only a small calibration set, yet aggressive 4-bit weight and 4-bit activation quantization (W4A4) remains challenging because large activation values enlarge quantization ranges and reduce the effective resolution available for ordinary values~\citep{frantar2023gptq,dettmers2022llmint8,xiao2023smoothquant}.
Existing methods address this difficulty in different ways: GPTQ-type methods optimize quantized weights to preserve full-precision layer outputs, scaling-based methods such as \smoothquant redistribute quantization difficulty between activations and weights, and rotation-based methods such as \quarot and \spinquant transform activation coordinates to improve low-bit quantization~\citep{frantar2023gptq,xiao2023smoothquant,ashkboos2024quarot,liu2025spinquant}.
Other hybrid and learned approaches combine or optimize these transformations in different ways~\citep{czako2026smoothrot,shao2024omniquant,sun2025flatquant}.
Despite their effectiveness, it remains unclear which components of the quantization error these operations address and how their roles are related.

In this paper, we address this gap through a local error-decomposition framework for weight-activation PTQ.
We exactly decompose the local quantization error into an activation-guided weight compensation (AGWC) term and an orthogonal residual.
For a fixed transformation, weight optimization can reduce the AGWC term but cannot change the orthogonal residual.
This decomposition therefore separates the part of the quantization error that can be addressed by weight compensation from the part that requires transformation design.
By further analyzing the orthogonal residual, we clarify how rotation and scaling act on different error structures and provide theoretical guidance for how existing weight optimization, rotation, and scaling techniques should be combined.

Our analysis further connects several design choices that are otherwise treated separately in existing quantization methods.
It explains the role of randomized rotation in controlling error associated with persistent activation outliers, and provides a principled basis for selecting among rotation patterns.
It also characterizes channel scaling through the balance between activation-side and weight-side quantities, relating second-moment-based scaling to \smoothquant-style scaling.
In this way, rotation, sign selection, and scaling can be understood as complementary operations guided by the structure of the residual error rather than as independent heuristics.

To examine the practical implications of these theoretical guidelines, we evaluate backpropagation-free configurations that follow the analysis, including signed online Hadamard rotation, calibration-based sign selection, and $L_2$ channel scaling.
Across eight Llama and Mistral models, these theoretically guided choices achieve performance competitive with gradient-trained \spinquant without backpropagation.
These results demonstrate that the proposed error-decomposition framework can provide practical guidance for configuring existing quantization techniques.

To summarize, our contributions are as follows:
\begin{itemize}
\item Using an exact decomposition of local weight-activation quantization error into an activation-guided weight compensation term and an orthogonal residual, we bound the residual through persistent channel-wise outlier and regular activation quantities, clarifying how transformation design complements weight compensation.

\item We analyze the orthogonal residual to provide theoretical guidance for randomized rotation, sign selection, and channel scaling, placing these operations within a common framework.

\item We derive a second-moment bound on the regular quantity and obtain an $L_2$ scaling rule from its Frobenius surrogate, relating it to \smoothquant-style $L_\infty$ scaling through a further relaxation.

\item We validate these design principles in backpropagation-free quantization configurations across multiple Llama and Mistral models, achieving performance competitive with gradient-trained \spinquant.
\end{itemize}

\section{Related Work}

\paragraph{Post-training weight-activation quantization.}
Post-training quantization (PTQ) reduces the memory and inference costs of large language models using only a small calibration set, without requiring full retraining~\citep{frantar2023gptq,lin2024awq,xiao2023smoothquant}.
Reconstruction-based methods such as GPTQ optimize quantized weights to preserve full-precision layer outputs~\citep{frantar2023gptq}, while activation-aware methods such as AWQ use activation statistics to identify and protect salient weight channels~\citep{lin2024awq}.
Compared with weight-only quantization, aggressive low-bit weight-activation quantization is more challenging because activation quantization introduces input-dependent error, and large activation values expand the quantization range and reduce the effective resolution available for ordinary values~\citep{dettmers2022llmint8,xiao2023smoothquant}.
In particular, persistent channel-wise outliers, namely channels that remain consistently large across tokens, can repeatedly dominate per-token activation quantization scales~\citep{xiao2023smoothquant, raman2025rethinking}.
Raman et al.~\citep{raman2025rethinking} further provide an explicit statistical criterion for detecting such channels and empirically analyze their formation through normalization and weight-matrix multiplication.
We therefore focus on persistent channel-wise outliers as a principal activation-side source of error in weight-activation quantization.

\paragraph{Output-preserving transformations.}
Outlier-aware quantization often applies transformations that preserve the full-precision output while modifying activations and weights before quantization.
\smoothquant transfers activation magnitudes to weights through output-preserving channel-wise scaling~\citep{xiao2023smoothquant}.
Rotation-based methods include \quip, which uses structured random orthogonal transformations to spread large weight entries and locally sensitive directions more evenly across coordinates before quantization~\citep{chee2023quip}, and \quipsharp, which efficiently implements a related preprocessing strategy using randomized Hadamard transforms~\citep{tseng2024quipsharp}.
\quarot applies Hadamard rotations throughout the Transformer~\citep{ashkboos2024quarot}, while \smoothrot combines Hadamard rotation with calibrated channel scaling~\citep{czako2026smoothrot}.
\dfrot empirically compares randomized orthogonal and Hadamard rotations under per-token activation quantization~\citep{xiang2025dfrot}.
\spinquant learns quantization-aware rotations through gradient-based calibration~\citep{liu2025spinquant}.
\omniquant~\citep{shao2024omniquant} and \flatquant~\citep{sun2025flatquant} learn more general equivalent transformations through gradient-based calibration.
Our work develops a common analytical basis for these transformations through local quantization-error decomposition.
The residual analysis explains how rotation and scaling complement weight compensation and guides their design without backpropagation.

\section{Problem Setting}
\label{sec:background_problem_setting}
We first define the local reconstruction problem for a linear layer under an output-preserving transformation.
We then introduce a simple decomposition of activations that separates persistent outlier structure from the remaining variation, which will be used to analyze rotation and scaling.

\subsection{Local Weight-Activation Quantization}
\label{sec:background_local_quantization}

Consider a linear layer $\Y=\X\W^\top$, where $\X\in\R^{T \times d}$ and $\W\in\R^{d'\times d}$.
For an invertible transformation $\Pmat\in\mathbb{R}^{d\times d}$, define $\Z=\X\Pmat$ and $\V=\W\Pmat^{-\top}$, so that $\X\W^\top=\Z\V^\top$.
Thus, $\Pmat$ changes the representation seen by the quantizers without changing the full-precision output.
We isolate activation quantization at the current layer, excluding upstream errors, to clarify the effects of rotation and scaling.

For activation quantization, let $b_A$ denote the quantized activation bit-width, and set $q_A = 2^{b_A-1}-1$, the maximum value after quantization.
For analytical tractability, we consider no-clipping dynamic per-token symmetric uniform quantization.
For each token $t$, define the dynamic scale by $\Delta_t = \| \Zentry{t}{:}\|_{\infty} / q_A$.
Then the quantized activation $\Ztilde=\quant_A(\Z)=\Z+\A$ is given element-wise by
\begin{align}
    \Ztildeentry{t}{j}
    =
    \quant_A(\Zentry{t}{j})
    =
    \Delta_t \operatorname{round} \left( \frac{\Zentry{t}{j}}{\Delta_t} \right),
    \qquad
    \Aentry{t}{j} = \Ztildeentry{t}{j} - \Zentry{t}{j}.
\end{align}
For weight quantization, we write $\Vtilde=\quant_W(\V)$.
Unlike the activation quantizer above, $\quant_W$ is not restricted to a fixed rounding rule.
It may itself be the result of an optimization procedure, such as GPTQ, or any other feasible weight quantization method.

We investigate the local reconstruction objective
\begin{align}
    \label{eq:local_quantization_error}
    \obj(\Pmat,\Vtilde)
    =
    \frac{1}{T}
    \left\|
    \Ztilde\Vtilde^{\top} - \Z\V^{\top}
    \right\|_F^2 .
\end{align}

This objective measures the output reconstruction error of a single linear layer due to quantization of its input activations and weights.

\subsection{Persistent Channel-Wise Outlier Decomposition}
\label{sec:background_co_decomposition}

Under the per-token activation quantizer introduced above, the quantization scale for token $t$ is determined by $\|\Zentry{t}{:}\|_\infty$, so a small number of consistently large activation channels can repeatedly dominate the quantization range and reduce the effective resolution available for the remaining coordinates.
Prior work has observed that activation outliers in LLMs often exhibit such a channel-wise persistent structure across tokens~\citep{xiao2023smoothquant,raman2025rethinking}.
Motivated by this observation, we represent each persistent outlier channel by a token-shared level and assign the remaining token-dependent variation to a regular component.
This decomposition serves as an analytical abstraction that separates a sparse persistent component, whose coordinates may coherently interfere after rotation, from the remaining activation variation, thereby enabling separate analyses of Hadamard rotation and channel-wise scaling.

To formalize this abstraction, let $\Cco(\X)\subseteq[d]$ be the set of persistent channel-wise outlier channels detected on calibration data.
For each CO channel $k\in\Cco(\X)$, let $L_k$ denote its persistent CO level.
We decompose the activation as $\X = \Xco + \Xreg$, where $\Xco$ is defined by
\begin{align}
    \label{eq:sparse_co_model}
    \Xcoentry{t}{k}
    =
    \begin{cases}
        L_k, & k\in\Cco(\X),\\
        0, & k\notin\Cco(\X),
    \end{cases}
    \qquad t\in[T],
\end{align}
and $\Xreg=\X-\Xco$.
We refer to this decomposition as the persistent sparse CO model.
Under this model, $\Xco$ captures the persistent activation component shared across tokens, while $\Xreg$ contains both the token-dependent deviations of the CO channels and the activations of the remaining channels.

We consider transformations of the form $\Pmat=\Lambdamat^{-1}\D\Hmat^\top$, where $\Lambdamat=\operatorname{diag}(\lambda_1,\ldots,\lambda_d)$ is a positive diagonal scaling matrix, with each $\lambda_k>0$ denoting the scaling coefficient for channel $k$, $\D$ is a diagonal sign matrix, and $\Hmat$ is the normalized Hadamard matrix.
The normalized Hadamard matrix is orthogonal, has entries of equal magnitude $1/\sqrt{d}$, and admits an $O(d\log d)$ fast transform; combining it with random diagonal signs gives the randomized Hadamard transform commonly used for coordinate flattening and incoherence processing ~\citep{ailon2006approx,tseng2024quipsharp}.
The diagonal scaling redistributes channel magnitudes between activations and weights, while the signed Hadamard transformation disperses activation coordinates across the transformed channel dimension.
Since the corresponding inverse transformation is absorbed into the weights, these operations preserve the full-precision linear-layer output.

\section{Error Decomposition and Residual Bounds for Quantization}
\label{sec:proposed_method}
We first derive an exact decomposition of the local objective in Section~\ref{sec:background_problem_setting} to separate the roles of weight compensation and transformation design.
We then use residual bounds to guide practical configurations: CO bounds motivate signed Hadamard rotation and sign sampling, while a second-moment surrogate of the regular bound yields an $L_2$ scaling rule.

\subsection{Local Quantization Error Decomposition and Residual Bound}
\label{sec:local_quantization_error_decomposition}
First, we exactly decompose the local quantization error and then bound the orthogonal residual to clarify the distinct roles of weight optimization, rotation, and scaling.
\begin{theorem}[Local quantization error decomposition and residual bound]
\label{thm:co_regular_error_decomposition}
Define $(\V^\star)^\top := \V^\top-\Ztilde^{\dagger}\A\V^\top$ and $\Pimat_{\Ztilde} := \Ztilde\Ztilde^\dagger$, where $\Pimat_{\Ztilde}$ is the orthogonal projector onto the column space of $\Ztilde$.
Then, the local quantization error $\obj(\Pmat,\Vtilde)$ in Eq.~\eqref{eq:local_quantization_error} admits the following exact decomposition and upper bound:
\begin{align}
    \obj(\Pmat,\Vtilde)
    \label{eq:local_error_projection_decomposition}
    &=
    \objagwc(\Pmat,\Vtilde)
    +
    \frac{1}{T}
    \left\|
    \left(
    \I
    -
    \Pimat_{\Ztilde}
    \right)
    \A\V^\top
    \right\|_F^2
    \\
    &\le
    \objagwc(\Pmat,\Vtilde)
    +
    \frac{d}{4q_A^2 T}
    \left(
    \sqrt{
    \left\|
    \W
    \Lambdamat
    \right\|_2^2
    \objco(\Lambdamat,\D)
    }
    +
    \sqrt{
    \objreg(\Lambdamat,\D)
    }
    \right)^2,
\end{align}
where
\begin{align}
    \label{eq:objagwc}
    \objagwc(\Pmat,\Vtilde)
    &=
    \frac{1}{T}
    \left\|
    \Ztilde
    \left(
    \Vtilde
    -
    \V^\star
    \right)^\top
    \right\|_F^2,
    \\
    \objco(\Lambdamat,\D)
    &=
    \sum_{t=1}^{T}
    \left\|
    \Xcoentry{t}{:}
    \Lambdamat^{-1}
    \D
    \Hmat^\top
    \right\|_\infty^2,
    \\
    \objreg(\Lambdamat,\D)
    &=
    \left\|
    \W
    \Lambdamat
    \right\|_2^2
    \sum_{t=1}^{T}
    \left\|
    \Xregentry{t}{:}
    \Lambdamat^{-1}
    \D
    \Hmat^\top
    \right\|_\infty^2.
\end{align}
\end{theorem}

The proof is provided in Appendix~\ref{app:proof_co_regular_error_decomposition}.

For full-column-rank $\Ztilde$, this projection identity matches CoreQ's mismatch decomposition~\citep{cha2026coreq}.
CoreQ uses this decomposition to limit overfitting to calibration data by adjusting how much the weights compensate for input errors.
We instead bound the transformation-dependent residual through persistent CO and regular activation quantities to clarify how rotation and scaling complement weight compensation.

We call $\objagwc$ in Eq.~\eqref{eq:objagwc} the \emph{activation-guided weight compensation} (AGWC) term: the activation discrepancy $\A$ guides the continuous target $\V^\star$, and $\objagwc$ measures the reconstruction error of the quantized weights relative to this target.
Geometrically, this term lies in the column space of $\widetilde{\Z}$ and can therefore be affected by changing the quantized weights, whereas the second term lies in the orthogonal complement of this space.
Thus, for a fixed transformation, weight optimization can reduce the AGWC term but cannot change the orthogonal residual.
The CO and regular quantities are not additional terms in the exact decomposition; they appear only in an upper bound on the orthogonal residual.
Theorem~\ref{thm:co_regular_error_decomposition} thus separates the scope of weight compensation from the residual error requiring transformation design.
This distinction guides how weight optimization, rotation, and scaling can be combined.

\subsection{Weight Compensation through the Lens of Error Decomposition}
\label{sec:weight_compensation_connections}

Theorem~\ref{thm:co_regular_error_decomposition} provides a common perspective on existing weight-reconstruction and compensation methods.
When $\A=\boldsymbol{0}$, $\V^\star=\V$ and $\objagwc$ reduces to the standard GPTQ-style reconstruction objective~\citep{frantar2023gptq}.
Thus, shared-input reconstruction can be viewed as the special case in which no activation discrepancy needs to be compensated.
When the input discrepancy is nonzero, however, shared-input reconstruction, even with quantized activations, does not explicitly compensate for this discrepancy.

QEP~\citep{arai2025qep} corrects weights before applying a quantizer such as GPTQ; with full compensation ($\alpha=1$), matching inputs and reference outputs, and an undamped invertible Hessian, its target equals $\V^\star$.
With GPTQ, the corrected weights are quantized sequentially, updating the remaining weights.
GPTAQ~\citep{li2025gptaq} incorporates input-error compensation into sequential updates derived from asymmetric reconstruction, using channel-wise residual components for efficiency.
For fixed inputs and $\Pmat$, Eq.~\eqref{eq:local_error_projection_decomposition} makes $\obj$ and $\objagwc$ equivalent up to a weight-independent constant over the same feasible set, without implying identical algorithmic outputs.
The orthogonal residual remains unchanged by weight compensation, so the CO and regular quantities in its bound guide complementary rotation and scaling.

In practice, we adopt backpropagation-free GPTAQ with activation quantization enabled before weight calibration to account for local and propagated input errors against full-precision reference outputs~\citep{li2025gptaq}.

\subsection{Randomized Hadamard Rotation for \texorpdfstring{$\objco$}{J-co}}
\label{sec:randomized_hadamard_rotation}

We now analyze the rotation-dependent CO quantity $\objco(\Lambdamat,\D)$ in the residual upper bound of Theorem~\ref{thm:co_regular_error_decomposition} and clarify how randomized Hadamard rotation controls this quantity.

\begin{theorem}[Hadamard rotation bounds]
\label{thm:fixed_randomized_sampled_hadamard}
Under the persistent sparse CO model defined in Eq.~\eqref{eq:sparse_co_model}, let
\begin{align}
    \Nco(\X)
    =
    \left|
    \Cco(\X)
    \right|,
    \qquad
    L_{\Lambda,\max}
    =
    \max_{k\in\Cco(\X)}
    \frac{|L_k|}{\lambda_k}.
\end{align}
Then, the following three bounds hold:

\paragraph{Fixed rotation.}
If $\D=\mathbf{I}$, then
\begin{align}
    \objco(\Lambdamat,\mathbf{I})
    \le
    \frac{
    T\Nco(\X)^2
    }{
    d
    }
    L_{\Lambda,\max}^2.
\end{align}

\paragraph{Randomized rotation.}
If $\D=\operatorname{diag}(\epsilon_1,\ldots,\epsilon_d)$ with independent Rademacher signs, then for any $\delta\in(0,1)$, with probability at least $1-\delta$,
\begin{align}
    \objco(\Lambdamat,\D)
    \le
    \frac{
    2T\Nco(\X)
    }{
    d
    }
    L_{\Lambda,\max}^2
    \left(
    \log(2d)
    +
    \log\frac{1}{\delta}
    \right).
\end{align}

\paragraph{Sampled rotation.}
Let $\mathcal{D}_{N_s}=\{\D^{(1)},\ldots,\D^{(N_s)}\}$ be a set of $N_s$ independently sampled random sign matrices.
Then, with probability at least $1-\delta$,
\begin{align}
    \min_{\D\in\mathcal{D}_{N_s}}
    \objco(\Lambdamat,\D)
    \le
    \frac{
    2T\Nco(\X)
    }{
    d
    }
    L_{\Lambda,\max}^2
    \left(
    \log(2d)
    +
    \frac{1}{N_s}
    \log\frac{1}{\delta}
    \right).
\end{align}
\end{theorem}

The proof is provided in Appendix~\ref{app:proof_fixed_randomized_sampled_hadamard}.

Theorem~\ref{thm:fixed_randomized_sampled_hadamard} identifies interference among persistent outlier channels as a mechanism governing the CO quantity in the residual bound.
Its fixed and randomized cases quantify this mechanism for existing Hadamard rotation schemes~\citep{tseng2024quipsharp,ashkboos2024quarot}; the sampled case bounds the best CO quantity in a candidate set and motivates calibration-based sign selection.
Because each entry of a normalized Hadamard matrix has magnitude $1/\sqrt{d}$, each CO channel contributes to every rotated coordinate at this scale; however, under a fixed rotation, contributions from multiple CO channels can constructively align, yielding the worst-case dependence $\Nco(\X)^2/d$.
Random signs convert these contributions into Rademacher sums and reduce this dependence to $\Nco(\X)\log d/d$ up to logarithmic factors, while sampling multiple sign matrices increases the probability that the candidate set contains a favorable interference pattern.
Appendix~\ref{app:interference_view} provides an interference-based interpretation of the $\order(\Nco(\X)^2)$ dependence under fixed rotation and its reduction to $\order(\Nco(\X))$ under random signs.
Since both the number of persistent CO channels $\Nco(\X)$ and their scaled maximum magnitude $L_{\Lambda,\max}$ depend on the model, layer, and activation distribution, these quantities provide one possible explanation for why different models exhibit different levels of accuracy degradation under low-bit quantization.

\subsection{Deriving \texorpdfstring{$L_2$}{L2} Scaling from a Bound on \texorpdfstring{$\objreg$}{J-reg}}
\label{sec:l2_scaling_regular_term}

We next consider the regular quantity $\objreg$ in the residual bound.
Unlike persistent COs, the regular component is not dominated by a small number of shared large coordinates.
After random-sign Hadamard rotation, its maximum coordinate can be controlled by its second-moment energy.

\begin{proposition}[$L_2$ scaling from second-moment balancing]
\label{prop:l2_scaling_regular_residual}
Let $\D$ be a random sign matrix with independent Rademacher signs.
For any $\delta\in(0,1)$, with probability at least $1-\delta$,
\begin{align}
    \objreg(\Lambdamat,\D)
    \le
    \frac{2\log(2dT/\delta)}{d}
    \left\|
        \W
        \Lambdamat
    \right\|_2^2
    \left\|
        \Xreg
        \Lambdamat^{-1}
    \right\|_F^2.
\end{align}
Consequently, the regular quantity $\objreg$ is controlled by $\left\|\W\Lambdamat\right\|_2^2\left\|\Xreg\Lambdamat^{-1}\right\|_F^2$.
Using the Frobenius surrogate $\mathcal{R}_2(\Lambdamat) = \left\| \W\Lambdamat \right\|_F^2 \left\| \Xreg\Lambdamat^{-1} \right\|_F^2$, we denote by $\lambda_k^{(2)}$ the channel-wise scaling coefficient obtained by minimizing $\mathcal{R}_2(\Lambdamat)$.
This gives
\begin{align}
    \lambda_k^{(2)}
    =
    \sqrt{
        \frac{
            \left\|
                \Xregentry{:}{k}
            \right\|_2
        }{
            \left\|
                \Wentry{:}{k}
            \right\|_2
        }
    }.
\end{align}
\end{proposition}

The proof is provided in Appendix~\ref{app:proof_l2_scaling_regular_residual}.

The resulting rule has the same $L_2$ balancing form as QuIP's rescaling~\citep{chee2023quip}.
Note that QuIP motivates this scaling through a heuristic surrogate for output error in weight-only quantization, whereas we derive it from a bound on the activation-quantization residual.
This derivation also enables a common interpretation of $L_2$ and \smoothquant-style $L_\infty$ scaling, as discussed in Section~\ref{sec:interpreting_smoothquant_linf_scaling}.

Proposition~\ref{prop:l2_scaling_regular_residual} explains why $L_2$ statistics are natural for the regular component.
The second-moment bound on $\objreg$ combines the activation-side energy
$\|\Xreg\Lambdamat^{-1}\|_F^2$ with the weight-side amplification factor
$\|\W\Lambdamat\|_2^2$.
In particular, the weight-side spectral norm
$\|\W\Lambdamat\|_2$ measures the worst-case amplification induced by the
scaled weight matrix and therefore reflects not only individual channel
magnitudes but also cross-channel interactions through the geometry of the
weight columns.
Directly optimizing the resulting objective can therefore account for such
interactions in principle.
However, the objective is not separable across channels, making it difficult
to obtain a simple closed-form channel-wise scaling rule.

To obtain a tractable channel-wise rule, we use the Frobenius surrogate
$\mathcal{R}_2(\Lambdamat)$.
This relaxation replaces the worst-case coupled amplification by the total second-moment energy of the scaled weight columns.
The surrogate is a product of sums of squared channel-wise quantities:
$\lambda_k\|\Wentry{:}{k}\|_2$ on the weight side and
$\|\Xregentry{:}{k}\|_2/\lambda_k$ on the activation side.
The equality condition in the Cauchy--Schwarz inequality yields the $L_2$ scaling rule in
Proposition~\ref{prop:l2_scaling_regular_residual}.
The proposition is stated in terms of the regular component $\Xreg$ to
characterize the error component targeted by the scaling rule.
In practice, since $\Xreg$ is not explicitly identified, we approximate its channel-wise second-moment statistics using the empirical full activations, as described in Section~\ref{sec:practical_implementation}.

\subsection{Comparing \texorpdfstring{$L_2$}{L2} Scaling with \smoothquant-Style \texorpdfstring{$L_\infty$}{L-inf}}
\label{sec:interpreting_smoothquant_linf_scaling}

Our residual analysis connects $L_2$ and \smoothquant-style $L_\infty$ scaling through successive relaxations of the same regular-term bound, providing a common interpretation not established by QuIP's heuristic derivation.

From Proposition~\ref{prop:l2_scaling_regular_residual}, after random-sign Hadamard rotation, the regular quantity satisfies $\objreg(\Lambdamat,\D) \lesssim \left\| \W \Lambdamat \right\|_2^2 \left\| \Xreg \Lambdamat^{-1} \right\|_F^2$.
Our $L_2$ scaling is obtained from a Frobenius surrogate of this second-moment bound.
In contrast, a \smoothquant-style interpretation can be obtained by further upper bounding the same quantity using $L_\infty$ channel statistics:
\begin{align}
    \left\|
    \W
    \Lambdamat
    \right\|_2^2
    \left\|
    \Xreg
    \Lambdamat^{-1}
    \right\|_F^2
    &\le
    \left\|
    \W
    \Lambdamat
    \right\|_F^2
    \left\|
    \Xreg
    \Lambdamat^{-1}
    \right\|_F^2
    \notag\\
    &\le
    d'T
    \left(
    \sum_{k=1}^{d}
    \lambda_k^2
    \left\|
    \Wentry{:}{k}
    \right\|_\infty^2
    \right)
    \left(
    \sum_{k=1}^{d}
    \frac{
    \left\|
    \Xregentry{:}{k}
    \right\|_\infty^2
    }{
    \lambda_k^2
    }
    \right).
\end{align}
This relaxation is looser because it first replaces the spectral norm of the scaled weight matrix by its Frobenius norm and then replaces channel-wise second-moment quantities by their maximum magnitudes.
We denote by $\lambda_k^{(\infty)}$ the channel-wise scaling coefficient
obtained by minimizing the resulting $L_\infty$ surrogate.
Then, up to a common positive scalar,
\begin{align}
    \lambda_k^{(\infty)}
    =
    \sqrt{
    \frac{
    \left\|
    \Xregentry{:}{k}
    \right\|_\infty
    }{
    \left\|
    \Wentry{:}{k}
    \right\|_\infty
    }
    }.
\end{align}
This rule is stated in terms of the regular component $\Xreg$ to
characterize the error component targeted by the scaling.
In practice, since $\Xreg$ is not explicitly identified, we approximate its
channel-wise maximum statistics using the empirical full activations,
recovering the familiar \smoothquant-style scaling rule.

Our framework thus explains both scaling rules through their retained statistics: the $L_2$ surrogate preserves channel-wise second moments, whereas the \smoothquant-style $L_\infty$ surrogate further relaxes them to channel maxima.
This provides a bound-based rationale for using $L_2$ statistics after randomized Hadamard rotation.

\subsection{From Theory to Practice: \proposal}
\label{sec:practical_implementation}

The preceding analysis leads to \proposal, a simple configuration with three components: add random signs to online Hadamard rotations, select sign patterns using calibration data, and apply $L_2$ channel scaling.
Theorem~\ref{thm:fixed_randomized_sampled_hadamard} motivates the first two, while Proposition~\ref{prop:l2_scaling_regular_residual} derives the scaling rule from a second-moment surrogate.
We implement these choices without backpropagation within the \quarot framework, following \smoothrot for scaling placement and fusion.

\paragraph{Signed online rotation (SOR).}
Theorem~\ref{thm:fixed_randomized_sampled_hadamard} shows that the sign structure of a Hadamard rotation is itself important: a fixed Hadamard transform can retain constructive interference among persistent CO channels, whereas random signs suppress this interference through Rademacher cancellation.
This suggests that the random-sign principle should be applied not only to offline rotations, but also to online Hadamard rotations that directly act on activations before quantization.
In the original \quarot, random signs are used only for offline rotation, whereas the online Hadamard rotations are fixed.
Motivated by this distinction, we insert sign flips before the online Hadamard rotations at the attention output and FFN down projections.
We also apply the same orthogonal signed rotation $\Rmat=\D\Hmat^\top$ to queries and keys within each head after RoPE, before key quantization, preserving unquantized attention scores since $(\Q\Rmat)(\Kmat\Rmat)^\top=\Q\Kmat^\top$.

\paragraph{Sign sampling (SS).}
The sampled bound in Theorem~\ref{thm:fixed_randomized_sampled_hadamard} motivates searching multiple sign patterns by controlling the best CO quantity in the candidate set.
Together with observed rotation variability~\citep{liu2025spinquant}, this motivates calibration-based selection without backpropagation.
We jointly select sign patterns for offline, attention-output, FFN-down, and QK rotations, generated from a single seed per candidate.
Each category shares its sign vector across layers, and QK additionally shares its vector across all query and key heads.
We screen candidates with RTN and select among the GPTAQ-reevaluated finalists by validation perplexity (Section~\ref{sec:experiments}).

\paragraph{$L_2$ scaling (L2S).}
We use the scaling rule derived by minimizing the second-moment surrogate in Proposition~\ref{prop:l2_scaling_regular_residual}, with empirical activation statistics as a proxy for the regular component.
Following \smoothrot, we apply scaling only to the FFN down-projection input.
For the down-projection input $\Xdown$ and weight $\Wdown$, we use the transformation $\Xdown\Wdown^\top = (\Xdown\Lambdamat^{-1}) (\Wdown\Lambdamat)^\top$.
Since the regular component is not explicitly separated in practice, we use the empirical full-activation norm as a readily computable proxy:
\begin{align}
    \lambda_k^{(2)}
    =
        \sqrt{
        \frac{
        \left\|
        \Xdownentry{:}{k}
        \right\|_2
    }{
        \left\|
        \Wdownentry{:}{k}
        \right\|_2
    }
    }.
\end{align}
As in \smoothrot, the scaling factors can be merged into the corresponding up- and down-projection weights and therefore introduce no additional inference-time overhead.

Overall, the decomposition in Theorem~\ref{thm:co_regular_error_decomposition} and the bounds in Theorem~\ref{thm:fixed_randomized_sampled_hadamard} and Proposition~\ref{prop:l2_scaling_regular_residual} provide a unified rationale for combining weight compensation with rotation and scaling.
Their effects are coupled, since changing the transformation can alter both terms of the exact decomposition.

\begin{table}[t]
\caption{
\label{tb:evaluation} W4A4 results with GPTAQ weight quantization and KV4 across model families and scales.
We report WikiText-2 perplexity (PPL) and average zero-shot accuracy over six tasks.
}
\centering
\renewcommand{\arraystretch}{1.15}
\setlength{\tabcolsep}{3.5pt}
\resizebox{\textwidth}{!}{
\begin{tabular}{c|c c|c c|c c|c c|c c} \hline
\multirow{3}{*}{Model}
& \multicolumn{2}{c|}{FP reference}
& \multicolumn{8}{c}{W4A4} \\ \cline{2-11}
& \multicolumn{2}{c|}{W16A16}
& \multicolumn{2}{c|}{\quarot}
& \multicolumn{2}{c|}{\smoothrot}
& \multicolumn{2}{c|}{\spinquant}
& \multicolumn{2}{c}{\proposal} \\
& PPL $\downarrow$ & Avg. $\uparrow$
& PPL $\downarrow$ & Avg. $\uparrow$
& PPL $\downarrow$ & Avg. $\uparrow$
& PPL $\downarrow$ & Avg. $\uparrow$
& PPL $\downarrow$ & Avg. $\uparrow$ \\ \hline
Llama-7B         & 5.67 & 69.39 & 6.12  & 66.17 & 6.13  & 66.16 & 6.19  & 66.18          & \textbf{6.09}  & \textbf{66.67} \\
Llama-13B        & 5.05 & 71.78 & 5.42  & 69.49 & 5.42  & 69.09 & 5.44  & 69.31          & \textbf{5.37}  & \textbf{69.91} \\ \hdashline
Llama2-7B        & 5.47 & 69.82 & 5.96  & 66.32 & 5.98  & 66.45 & 6.06  & 65.98          & \textbf{5.94}  & \textbf{66.59} \\
Llama2-13B       & 4.86 & 72.57 & 5.24  & 69.87 & 5.26  & 70.15 & 5.26  & 70.23          & \textbf{5.23}  & \textbf{70.50} \\ \hdashline
Llama3-8B        & 5.94 & 73.30 & 7.45  & 65.71 & 7.57  & 65.85 & 7.33  & \textbf{68.36} & \textbf{7.20}  & 68.10          \\
Llama3.2-1B      & 9.41 & 59.80 & 13.81 & 49.58 & 14.13 & 48.63 & 13.33 & 50.37          & \textbf{12.92} & \textbf{50.71} \\
Llama3.2-3B      & 7.55 & 68.20 & 9.27  & 61.12 & 9.35  & 59.91 & 9.19  & \textbf{62.03} & \textbf{9.08}  & 60.94          \\ \hdashline
Mistral-7B-v0.3  & 5.35 & 73.73 & 5.73  & 70.55 & 5.77  & 69.80 & 5.74  & \textbf{71.17} & \textbf{5.71}  & 71.00          \\ \hline
\end{tabular}
}
\end{table}

\section{Experiments}
\label{sec:experiments}

\paragraph{Tasks and baselines.}
We evaluate Llama-7B/13B~\citep{touvron2023llama}, Llama2-7B/13B~\citep{touvron2023llama2}, Llama3-8B~\citep{grattafiori2024llama3}, Llama3.2-1B/3B~\citep{meta2024llama32}, and Mistral-7B-v0.3~\citep{jiang2023mistral,mistralai2024mistral7bv03} under W4A4 quantization.
We evaluate perplexity (PPL) on $64$ samples each from WikiText-2 (WT2) test and a C4 validation subset.
For zero-shot evaluation, we report accuracy on PIQA~\citep{bisk2020piqa}, ARC-Easy (ARC-E) and ARC-Challenge (ARC-C)~\citep{clark2018arc}, HellaSwag (HS)~\citep{zellers2019hellaswag}, WinoGrande (WG)~\citep{sakaguchi2021winogrande}, and LAMBADA (LB)~\citep{paperno2016lambada}.
We compare \proposal with \quarot, \smoothrot, and \spinquant~\citep{liu2025spinquant}, applying the same GPTAQ weight-quantization settings to the transformed models produced by all four methods.

\paragraph{Quantization and calibration.}
We simulate W4A4 decoder linear layers with KV4: weights use per-output-channel symmetric quantization, activations per-token symmetric quantization, and K/V per-token, per-head asymmetric quantization.
Activation and K/V clipping ratios are $0.9$ and $0.95$.
GPTAQ uses $128$ WikiText-2 train samples.
Calibration includes local activation and propagated weight, activation, and K/V errors against corresponding unquantized-path outputs.
Scaling statistics use the untransformed, unquantized model with $512$ WikiText-2 train samples.
Following the rotation-training protocol of \spinquant~\citep{liu2025spinquant}, we train rotations on WikiText-2 train for $100$ steps with frozen 16-bit weights and 4-bit activations and K/V, then quantize the rotated weights with GPTAQ.

\paragraph{Rotation selection.}
For \proposal, we screen $10$ candidates with RTN, reevaluate the top $3$ with GPTAQ, and select the lowest PPL on WikiText-2 validation.

\paragraph{Main results.}
Table~\ref{tb:evaluation} summarizes the main W4A4 results across eight models from two families and from 1B to 13B parameters; Appendix~\ref{app:full_results} reports C4 perplexity and per-task accuracies.
\proposal achieves the lowest WikiText-2 perplexity on all eight models and the highest average zero-shot accuracy on five.
It also achieves the lowest C4 perplexity on every model, indicating that the perplexity improvements are not limited to WikiText-2, which is used for calibration and rotation selection.
On Llama3-8B and Mistral-7B-v0.3, its average accuracy is within $0.26$ and $0.17$ percentage points of \spinquant, respectively; on Llama3.2-3B, it improves perplexity over \spinquant while trailing its average accuracy by $1.09$ points.
Overall, these results demonstrate the practical value of the transformation design guided by our analysis, with consistent perplexity gains and competitive downstream accuracy across model families and scales without backpropagation.

\begin{figure}[t]
    \centering
    \begin{minipage}[t]{0.46\linewidth}
        \vspace{0pt}
        \centering
        \includegraphics[
            width=\linewidth
        ]{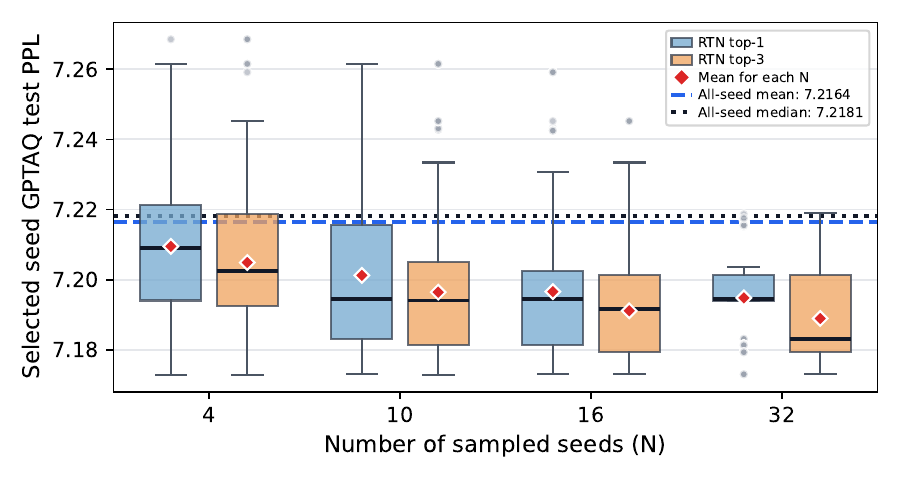}
        \captionof{figure}{
            RTN-based seed-selection ablation on Llama3-8B with GPTAQ.
        }
        \label{fig:sampling_stability}
    \end{minipage}
    \hfill
    \begin{minipage}[t]{0.51\linewidth}
        \vspace{0pt}
        \centering
        \renewcommand{\arraystretch}{1.15}
        \captionof{table}{
    Component ablation on Llama3-8B with GPTAQ (W4A4KV4).
    Parentheses indicate the residual-bound term motivating each component.
}
\label{tb:ablation_gptaq}
\resizebox{\linewidth}{!}{
    \begin{tabular}{c c c c|c} \hline
    \multicolumn{4}{c|}{Method}                                  & Llama3-8B        \\ \cline{1-4}
    \quarot    & SOR ($\objco$) & SS ($\objco$) & L2S ($\objreg$)& PPL $\downarrow$ \\ \hline
    \checkmark &                &               &                & 7.45             \\ \cdashline{1-5}[2pt/1pt]
    \checkmark & \checkmark     &               &                & 7.30             \\
    \checkmark &                & \checkmark    &                & 7.40             \\ 
    \checkmark &                &               & \checkmark     & 7.32             \\ \cdashline{1-5}[2pt/1pt]
    \checkmark & \checkmark     & \checkmark    & \checkmark     & \textbf{7.20}             \\ \hline
    \end{tabular}
}

    \end{minipage}
\end{figure}

\paragraph{Component ablation.}
Table~\ref{tb:ablation_gptaq} ablates the proposed components on Llama3-8B under W4A4 quantization with GPTAQ.
Starting from \quarot ($7.45$), adding SOR (including QK rotation), SS, or L2S individually reduces WikiText-2 perplexity to $7.30$, $7.40$, and $7.32$, respectively, with SOR providing the largest individual gain.
Combining all components further reduces perplexity to $7.20$.
These gains are consistent with their complementary motivations: SOR and SS primarily target $\objco$, whereas L2S targets $\objreg$, although their effects on quantization error need not be independent.

\paragraph{Seed-selection ablation.}
On Llama3-8B with W4A4, KV4, and GPTAQ, we precompute RTN validation and GPTAQ validation/test perplexities on WikiText-2 for $100$ random rotation seeds.
For each $N\in\{4,10,16,32\}$, we repeat $1000$ trials, sampling $N$ seeds without replacement and applying both selection rules to the same subset.
Top-$1$ selects the RTN validation winner; top-$3$ reranks the three best RTN candidates by GPTAQ validation perplexity.
Figure~\ref{fig:sampling_stability} shows that both rules improve mean GPTAQ test perplexity over uniform selection (the all-seed mean) and benefit from larger candidate sets.
Top-$3$ consistently outperforms top-$1$, whose gains diminish between $N=16$ and $N=32$.
Thus, RTN screening offers a low-cost selection criterion: top-$1$ needs one GPTAQ run, while top-$3$ trades three runs for better selection.

\section{Conclusion}

In this paper, we analyzed local weight-activation quantization error through an exact decomposition into an activation-guided weight compensation term and an orthogonal residual.
Bounding the residual through persistent channel-wise outlier and regular activation quantities clarifies how rotation and scaling complement weight optimization.
Within this framework, our bounds characterize outlier interference under fixed, randomized, and sampled Hadamard rotations, explaining the role of random signs in suppressing constructive interference and motivating calibration-based sign selection.
For scaling, minimizing a second-moment surrogate of the regular-term bound yields an $L_2$ rule, while a further relaxation recovers a \smoothquant-style $L_\infty$ rule.
We evaluated these design principles through backpropagation-free configurations combining signed online rotation, sign sampling, and $L_2$ channel scaling.
Experiments across eight Llama and Mistral models demonstrated performance competitive with gradient-trained \spinquant without backpropagation.
Our analysis provides a principled basis for designing output-preserving transformations for low-bit weight-activation quantization.

\section*{AI Use Statement}
Generative AI tools were used to assist with writing and polishing the manuscript.
We reviewed and revised all AI-assisted text.

\bibliography{references}

@inproceedings{xiao2023smoothquant,
    title = 	 {{S}mooth{Q}uant: Accurate and Efficient Post-Training Quantization for Large Language Models},
    author =       {Xiao, Guangxuan and Lin, Ji and Seznec, Mickael and Wu, Hao and Demouth, Julien and Han, Song},
    booktitle = 	 {Proceedings of the 40th International Conference on Machine Learning},
    pages = 	 {38087--38099},
    year = 	 {2023},
    editor = 	 {Krause, Andreas and Brunskill, Emma and Cho, Kyunghyun and Engelhardt, Barbara and Sabato, Sivan and Scarlett, Jonathan},
    volume = 	 {202},
    series = 	 {Proceedings of Machine Learning Research},
    month = 	 {23--29 Jul},
    publisher =    {PMLR},
    url = 	 {https://proceedings.mlr.press/v202/xiao23c.html}
}

@inproceedings{dettmers2022llmint8,
 author = {Dettmers, Tim and Lewis, Mike and Belkada, Younes and Zettlemoyer, Luke},
 booktitle = {Advances in Neural Information Processing Systems},
 editor = {S. Koyejo and S. Mohamed and A. Agarwal and D. Belgrave and K. Cho and A. Oh},
 pages = {30318--30332},
 publisher = {Curran Associates, Inc.},
 title = {GPT3.int8(): 8-bit Matrix Multiplication for Transformers at Scale},
 url = {https://proceedings.neurips.cc/paper_files/paper/2022/file/c3ba4962c05c49636d4c6206a97e9c8a-Paper-Conference.pdf},
 volume = {35},
 year = {2022}
}

@misc{raman2025rethinking,
    title={Rethinking the Outlier Distribution in Large Language Models: An In-depth Study}, 
    author={Rahul Raman and Khushi Sharma and Sai Qian Zhang},
    year={2025},
    eprint={2505.21670},
    archivePrefix={arXiv},
    primaryClass={cs.CL},
    url={https://arxiv.org/abs/2505.21670}, 
}

@inproceedings{ashkboos2024quarot,
 author = {Ashkboos, Saleh and Mohtashami, Amirkeivan and Croci, Maximilian L. and Li, Bo and Cameron, Pashmina and Jaggi, Martin and Alistarh, Dan and Hoefler, Torsten and Hensman, James},
 booktitle = {Advances in Neural Information Processing Systems},
 doi = {10.52202/079017-3180},
 editor = {A. Globerson and L. Mackey and D. Belgrave and A. Fan and U. Paquet and J. Tomczak and C. Zhang},
 pages = {100213--100240},
 publisher = {Curran Associates, Inc.},
 title = {QuaRot: Outlier-Free 4-Bit Inference in Rotated LLMs},
 url = {https://proceedings.neurips.cc/paper_files/paper/2024/file/b5b939436789f76f08b9d0da5e81af7c-Paper-Conference.pdf},
 volume = {37},
 year = {2024}
}

@inproceedings{tseng2024quipsharp,
    title = 	 {{Q}u{IP}\#: Even Better {LLM} Quantization with Hadamard Incoherence and Lattice Codebooks},
    author =       {Tseng, Albert and Chee, Jerry and Sun, Qingyao and Kuleshov, Volodymyr and De Sa, Christopher},
    booktitle = 	 {Proceedings of the 41st International Conference on Machine Learning},
    pages = 	 {48630--48656},
    year = 	 {2024},
    editor = 	 {Salakhutdinov, Ruslan and Kolter, Zico and Heller, Katherine and Weller, Adrian and Oliver, Nuria and Scarlett, Jonathan and Berkenkamp, Felix},
    volume = 	 {235},
    series = 	 {Proceedings of Machine Learning Research},
    month = 	 {21--27 Jul},
    publisher =    {PMLR},
    url = 	 {https://proceedings.mlr.press/v235/tseng24a.html}
}

@misc{cha2026coreq,
    title = {{CoreQ}: Learning-Free Mismatch Correction and Successive Rounding for Quantization},
    author = {Cha, Seohyeon and Chen, Huancheng and Kim, Dongjun and Zhang, Haoran and Chan, Kevin and de Veciana, Gustavo and Vikalo, Haris},
    year = {2026},
    eprint = {2602.05902},
    archivePrefix = {arXiv},
    primaryClass = {cs.LG},
    url = {https://arxiv.org/abs/2602.05902v2},
    note = {Version 2}
}

@inproceedings{chee2023quip,
    author = {Chee, Jerry and Cai, Yaohui and Kuleshov, Volodymyr and De Sa, Christopher},
    booktitle = {Advances in Neural Information Processing Systems},
    doi = {10.52202/075280-0196},
    editor = {A. Oh and T. Naumann and A. Globerson and K. Saenko and M. Hardt and S. Levine},
    pages = {4396--4429},
    publisher = {Curran Associates, Inc.},
    title = {QuIP: 2-Bit Quantization of Large Language Models With Guarantees},
    url = {https://proceedings.neurips.cc/paper_files/paper/2023/file/0df38cd13520747e1e64e5b123a78ef8-Paper-Conference.pdf},
    volume = {36},
    year = {2023}
}

@inproceedings{shao2024omniquant,
    author = {Shao, Wenqi and Chen, Mengzhao and Zhang, Zhaoyang and Xu, Peng and Zhao, Lirui and Li, Zhiqian and Zhang, Kaipeng and Peng, Gao and Qiao, Yu and Luo, Ping},
    booktitle = {International Conference on Learning Representations},
    editor = {B. Kim and Y. Yue and S. Chaudhuri and K. Fragkiadaki and M. Khan and Y. Sun},
    pages = {45472--45496},
    title = {OmniQuant: Omnidirectionally Calibrated Quantization for Large Language Models},
    url = {https://proceedings.iclr.cc/paper_files/paper/2024/file/c6483c8a68083af3383f91ee0dc6db95-Paper-Conference.pdf},
    volume = {2024},
    year = {2024}
}

@inproceedings{liu2025spinquant,
    author = {Liu, Zechun and Zhao, Changsheng and Fedorov, Igor and Soran, Bilge and Choudhary, Dhruv and Krishnamoorthi, Raghuraman and Chandra, Vikas and Tian, Yuandong and Blankevoort, Tijmen},
    booktitle = {International Conference on Learning Representations},
    editor = {Y. Yue and A. Garg and N. Peng and F. Sha and R. Yu},
    pages = {92009--92032},
    title = {SpinQuant: LLM Quantization with Learned Rotations},
    url = {https://proceedings.iclr.cc/paper_files/paper/2025/file/e5b1c0d4866f72393c522c8a00eed4eb-Paper-Conference.pdf},
    volume = {2025},
    year = {2025}
}

@inproceedings{sun2025flatquant,
    title = 	 {{F}lat{Q}uant: Flatness Matters for {LLM} Quantization},
    author =       {Sun, Yuxuan and Liu, Ruikang and Bai, Haoli and Bao, Han and Zhao, Kang and Li, Yuening and Hu, Jiaxin and Yu, Xianzhi and Hou, Lu and Yuan, Chun and Jiang, Xin and Liu, Wulong and Yao, Jun},
    booktitle = 	 {Proceedings of the 42nd International Conference on Machine Learning},
    pages = 	 {57587--57613},
    year = 	 {2025},
    editor = 	 {Singh, Aarti and Fazel, Maryam and Hsu, Daniel and Lacoste-Julien, Simon and Berkenkamp, Felix and Maharaj, Tegan and Wagstaff, Kiri and Zhu, Jerry},
    volume = 	 {267},
    series = 	 {Proceedings of Machine Learning Research},
    month = 	 {13--19 Jul},
    publisher =    {PMLR},
    url = 	 {https://proceedings.mlr.press/v267/sun25l.html}
}

@misc{frantar2023gptq,
    title={GPTQ: Accurate Post-Training Quantization for Generative Pre-trained Transformers}, 
    author={Elias Frantar and Saleh Ashkboos and Torsten Hoefler and Dan Alistarh},
    year={2023},
    eprint={2210.17323},
    archivePrefix={arXiv},
    primaryClass={cs.LG},
    url={https://arxiv.org/abs/2210.17323}, 
}

@inproceedings{lin2024awq,
    author = {Lin, Ji and Tang, Jiaming and Tang, Haotian and Yang, Shang and Chen, Wei-Ming and Wang, Wei-Chen and Xiao, Guangxuan and Dang, Xingyu and Gan, Chuang and Han, Song},
    booktitle = {Proceedings of Machine Learning and Systems},
    editor = {P. Gibbons and G. Pekhimenko and C. De Sa},
    pages = {87--100},
    title = {AWQ: Activation-aware Weight Quantization for On-Device LLM Compression and Acceleration},
    url = {https://proceedings.mlsys.org/paper_files/paper/2024/file/42a452cbafa9dd64e9ba4aa95cc1ef21-Paper-Conference.pdf},
    volume = {6},
    year = {2024}
}

@inproceedings{czako2026smoothrot,
    author={Czakó, Patrik and Kertész, Gábor and Szénási, Sándor},
    booktitle={2025 IEEE International Conference on Systems, Man, and Cybernetics (SMC)}, 
    title={SmoothRot: Combining Channel-Wise Scaling and Rotation for Quantization-Friendly LLMs}, 
    year={2025},
    volume={},
    number={},
    pages={6461-6466},
    doi={10.1109/SMC58881.2025.11342731}
}

@inproceedings{bisk2020piqa,
    title={Piqa: Reasoning about physical commonsense in natural language},
    author={Bisk, Yonatan and Zellers, Rowan and Gao, Jianfeng and Choi, Yejin and others},
    booktitle={Proceedings of the AAAI conference on artificial intelligence},
    volume={34},
    number={05},
    pages={7432--7439},
    year={2020}
}

@misc{clark2018arc,
      title={Think you have Solved Question Answering? Try ARC, the AI2 Reasoning Challenge}, 
      author={Peter Clark and Isaac Cowhey and Oren Etzioni and Tushar Khot and Ashish Sabharwal and Carissa Schoenick and Oyvind Tafjord},
      year={2018},
      eprint={1803.05457},
      archivePrefix={arXiv},
      primaryClass={cs.AI},
      url={https://arxiv.org/abs/1803.05457}, 
}

@inproceedings{zellers2019hellaswag,
    title = "{H}ella{S}wag: Can a Machine Really Finish Your Sentence?",
    author = "Zellers, Rowan  and
      Holtzman, Ari  and
      Bisk, Yonatan  and
      Farhadi, Ali  and
      Choi, Yejin",
    editor = "Korhonen, Anna  and
      Traum, David  and
      M{\`a}rquez, Llu{\'i}s",
    booktitle = "Proceedings of the 57th Annual Meeting of the Association for Computational Linguistics",
    month = jul,
    year = "2019",
    address = "Florence, Italy",
    publisher = "Association for Computational Linguistics",
    url = "https://aclanthology.org/P19-1472/",
    doi = "10.18653/v1/P19-1472",
    pages = "4791--4800"
}

@article{sakaguchi2021winogrande,
    author = {Sakaguchi, Keisuke and Bras, Ronan Le and Bhagavatula, Chandra and Choi, Yejin},
    title = {WinoGrande: an adversarial winograd schema challenge at scale},
    year = {2021},
    issue_date = {September 2021},
    publisher = {Association for Computing Machinery},
    address = {New York, NY, USA},
    volume = {64},
    number = {9},
    issn = {0001-0782},
    url = {https://doi.org/10.1145/3474381},
    doi = {10.1145/3474381},
    journal = {Commun. ACM},
    month = aug,
    pages = {99–106},
    numpages = {8}
}

@inproceedings{paperno2016lambada,
    title = "The {LAMBADA} dataset: Word prediction requiring a broad discourse context",
    author = "Paperno, Denis  and
      Kruszewski, Germ{\'a}n  and
      Lazaridou, Angeliki  and
      Pham, Ngoc Quan  and
      Bernardi, Raffaella  and
      Pezzelle, Sandro  and
      Baroni, Marco  and
      Boleda, Gemma  and
      Fern{\'a}ndez, Raquel",
    editor = "Erk, Katrin  and
      Smith, Noah A.",
    booktitle = "Proceedings of the 54th Annual Meeting of the Association for Computational Linguistics (Volume 1: Long Papers)",
    month = aug,
    year = "2016",
    address = "Berlin, Germany",
    publisher = "Association for Computational Linguistics",
    url = "https://aclanthology.org/P16-1144/",
    doi = "10.18653/v1/P16-1144",
    pages = "1525--1534"
}

@misc{grattafiori2024llama3,
      title={The Llama 3 Herd of Models}, 
      author={Aaron Grattafiori and Abhimanyu Dubey and Abhinav Jauhri and Abhinav Pandey and Abhishek Kadian and Ahmad Al-Dahle and Aiesha Letman and Akhil Mathur and Alan Schelten and Alex Vaughan and Amy Yang and Angela Fan and Anirudh Goyal and Anthony Hartshorn and Aobo Yang and Archi Mitra and Archie Sravankumar and Artem Korenev and Arthur Hinsvark and Arun Rao and Aston Zhang and Aurelien Rodriguez and Austen Gregerson and Ava Spataru and Baptiste Roziere and Bethany Biron and Binh Tang and Bobbie Chern and Charlotte Caucheteux and Chaya Nayak and Chloe Bi and Chris Marra and Chris McConnell and Christian Keller and Christophe Touret and Chunyang Wu and Corinne Wong and Cristian Canton Ferrer and Cyrus Nikolaidis and Damien Allonsius and Daniel Song and Danielle Pintz and Danny Livshits and Danny Wyatt and David Esiobu and Dhruv Choudhary and Dhruv Mahajan and Diego Garcia-Olano and Diego Perino and Dieuwke Hupkes and Egor Lakomkin and Ehab AlBadawy and Elina Lobanova and Emily Dinan and Eric Michael Smith and Filip Radenovic and Francisco Guzmán and Frank Zhang and Gabriel Synnaeve and Gabrielle Lee and Georgia Lewis Anderson and Govind Thattai and Graeme Nail and Gregoire Mialon and Guan Pang and Guillem Cucurell and Hailey Nguyen and Hannah Korevaar and Hu Xu and Hugo Touvron and Iliyan Zarov and Imanol Arrieta Ibarra and Isabel Kloumann and Ishan Misra and Ivan Evtimov and Jack Zhang and Jade Copet and Jaewon Lee and Jan Geffert and Jana Vranes and Jason Park and Jay Mahadeokar and Jeet Shah and Jelmer van der Linde and Jennifer Billock and Jenny Hong and Jenya Lee and Jeremy Fu and Jianfeng Chi and Jianyu Huang and Jiawen Liu and Jie Wang and Jiecao Yu and Joanna Bitton and Joe Spisak and Jongsoo Park and Joseph Rocca and Joshua Johnstun and Joshua Saxe and Junteng Jia and Kalyan Vasuden Alwala and Karthik Prasad and Kartikeya Upasani and Kate Plawiak and Ke Li and Kenneth Heafield and Kevin Stone and Khalid El-Arini and Krithika Iyer and Kshitiz Malik and Kuenley Chiu and Kunal Bhalla and Kushal Lakhotia and Lauren Rantala-Yeary and Laurens van der Maaten and Lawrence Chen and Liang Tan and Liz Jenkins and Louis Martin and Lovish Madaan and Lubo Malo and Lukas Blecher and Lukas Landzaat and Luke de Oliveira and Madeline Muzzi and Mahesh Pasupuleti and Mannat Singh and Manohar Paluri and Marcin Kardas and Maria Tsimpoukelli and Mathew Oldham and Mathieu Rita and Maya Pavlova and Melanie Kambadur and Mike Lewis and Min Si and Mitesh Kumar Singh and Mona Hassan and Naman Goyal and Narjes Torabi and Nikolay Bashlykov and Nikolay Bogoychev and Niladri Chatterji and Ning Zhang and Olivier Duchenne and Onur Çelebi and Patrick Alrassy and Pengchuan Zhang and Pengwei Li and Petar Vasic and Peter Weng and Prajjwal Bhargava and Pratik Dubal and Praveen Krishnan and Punit Singh Koura and Puxin Xu and Qing He and Qingxiao Dong and Ragavan Srinivasan and Raj Ganapathy and Ramon Calderer and Ricardo Silveira Cabral and Robert Stojnic and Roberta Raileanu and Rohan Maheswari and Rohit Girdhar and Rohit Patel and Romain Sauvestre and Ronnie Polidoro and Roshan Sumbaly and Ross Taylor and Ruan Silva and Rui Hou and Rui Wang and Saghar Hosseini and Sahana Chennabasappa and Sanjay Singh and Sean Bell and Seohyun Sonia Kim and Sergey Edunov and Shaoliang Nie and Sharan Narang and Sharath Raparthy and Sheng Shen and Shengye Wan and Shruti Bhosale and Shun Zhang and Simon Vandenhende and Soumya Batra and Spencer Whitman and Sten Sootla and Stephane Collot and Suchin Gururangan and Sydney Borodinsky and Tamar Herman and Tara Fowler and Tarek Sheasha and Thomas Georgiou and Thomas Scialom and Tobias Speckbacher and Todor Mihaylov and Tong Xiao and Ujjwal Karn and Vedanuj Goswami and Vibhor Gupta and Vignesh Ramanathan and Viktor Kerkez and Vincent Gonguet and Virginie Do and Vish Vogeti and Vítor Albiero and Vladan Petrovic and Weiwei Chu and Wenhan Xiong and Wenyin Fu and Whitney Meers and Xavier Martinet and Xiaodong Wang and Xiaofang Wang and Xiaoqing Ellen Tan and Xide Xia and Xinfeng Xie and Xuchao Jia and Xuewei Wang and Yaelle Goldschlag and Yashesh Gaur and Yasmine Babaei and Yi Wen and Yiwen Song and Yuchen Zhang and Yue Li and Yuning Mao and Zacharie Delpierre Coudert and Zheng Yan and Zhengxing Chen and Zoe Papakipos and Aaditya Singh and Aayushi Srivastava and Abha Jain and Adam Kelsey and Adam Shajnfeld and Adithya Gangidi and Adolfo Victoria and Ahuva Goldstand and Ajay Menon and Ajay Sharma and Alex Boesenberg and Alexei Baevski and Allie Feinstein and Amanda Kallet and Amit Sangani and Amos Teo and Anam Yunus and Andrei Lupu and Andres Alvarado and Andrew Caples and Andrew Gu and Andrew Ho and Andrew Poulton and Andrew Ryan and Ankit Ramchandani and Annie Dong and Annie Franco and Anuj Goyal and Aparajita Saraf and Arkabandhu Chowdhury and Ashley Gabriel and Ashwin Bharambe and Assaf Eisenman and Azadeh Yazdan and Beau James and Ben Maurer and Benjamin Leonhardi and Bernie Huang and Beth Loyd and Beto De Paola and Bhargavi Paranjape and Bing Liu and Bo Wu and Boyu Ni and Braden Hancock and Bram Wasti and Brandon Spence and Brani Stojkovic and Brian Gamido and Britt Montalvo and Carl Parker and Carly Burton and Catalina Mejia and Ce Liu and Changhan Wang and Changkyu Kim and Chao Zhou and Chester Hu and Ching-Hsiang Chu and Chris Cai and Chris Tindal and Christoph Feichtenhofer and Cynthia Gao and Damon Civin and Dana Beaty and Daniel Kreymer and Daniel Li and David Adkins and David Xu and Davide Testuggine and Delia David and Devi Parikh and Diana Liskovich and Didem Foss and Dingkang Wang and Duc Le and Dustin Holland and Edward Dowling and Eissa Jamil and Elaine Montgomery and Eleonora Presani and Emily Hahn and Emily Wood and Eric-Tuan Le and Erik Brinkman and Esteban Arcaute and Evan Dunbar and Evan Smothers and Fei Sun and Felix Kreuk and Feng Tian and Filippos Kokkinos and Firat Ozgenel and Francesco Caggioni and Frank Kanayet and Frank Seide and Gabriela Medina Florez and Gabriella Schwarz and Gada Badeer and Georgia Swee and Gil Halpern and Grant Herman and Grigory Sizov and Guangyi and Zhang and Guna Lakshminarayanan and Hakan Inan and Hamid Shojanazeri and Han Zou and Hannah Wang and Hanwen Zha and Haroun Habeeb and Harrison Rudolph and Helen Suk and Henry Aspegren and Hunter Goldman and Hongyuan Zhan and Ibrahim Damlaj and Igor Molybog and Igor Tufanov and Ilias Leontiadis and Irina-Elena Veliche and Itai Gat and Jake Weissman and James Geboski and James Kohli and Janice Lam and Japhet Asher and Jean-Baptiste Gaya and Jeff Marcus and Jeff Tang and Jennifer Chan and Jenny Zhen and Jeremy Reizenstein and Jeremy Teboul and Jessica Zhong and Jian Jin and Jingyi Yang and Joe Cummings and Jon Carvill and Jon Shepard and Jonathan McPhie and Jonathan Torres and Josh Ginsburg and Junjie Wang and Kai Wu and Kam Hou U and Karan Saxena and Kartikay Khandelwal and Katayoun Zand and Kathy Matosich and Kaushik Veeraraghavan and Kelly Michelena and Keqian Li and Kiran Jagadeesh and Kun Huang and Kunal Chawla and Kyle Huang and Lailin Chen and Lakshya Garg and Lavender A and Leandro Silva and Lee Bell and Lei Zhang and Liangpeng Guo and Licheng Yu and Liron Moshkovich and Luca Wehrstedt and Madian Khabsa and Manav Avalani and Manish Bhatt and Martynas Mankus and Matan Hasson and Matthew Lennie and Matthias Reso and Maxim Groshev and Maxim Naumov and Maya Lathi and Meghan Keneally and Miao Liu and Michael L. Seltzer and Michal Valko and Michelle Restrepo and Mihir Patel and Mik Vyatskov and Mikayel Samvelyan and Mike Clark and Mike Macey and Mike Wang and Miquel Jubert Hermoso and Mo Metanat and Mohammad Rastegari and Munish Bansal and Nandhini Santhanam and Natascha Parks and Natasha White and Navyata Bawa and Nayan Singhal and Nick Egebo and Nicolas Usunier and Nikhil Mehta and Nikolay Pavlovich Laptev and Ning Dong and Norman Cheng and Oleg Chernoguz and Olivia Hart and Omkar Salpekar and Ozlem Kalinli and Parkin Kent and Parth Parekh and Paul Saab and Pavan Balaji and Pedro Rittner and Philip Bontrager and Pierre Roux and Piotr Dollar and Polina Zvyagina and Prashant Ratanchandani and Pritish Yuvraj and Qian Liang and Rachad Alao and Rachel Rodriguez and Rafi Ayub and Raghotham Murthy and Raghu Nayani and Rahul Mitra and Rangaprabhu Parthasarathy and Raymond Li and Rebekkah Hogan and Robin Battey and Rocky Wang and Russ Howes and Ruty Rinott and Sachin Mehta and Sachin Siby and Sai Jayesh Bondu and Samyak Datta and Sara Chugh and Sara Hunt and Sargun Dhillon and Sasha Sidorov and Satadru Pan and Saurabh Mahajan and Saurabh Verma and Seiji Yamamoto and Sharadh Ramaswamy and Shaun Lindsay and Shaun Lindsay and Sheng Feng and Shenghao Lin and Shengxin Cindy Zha and Shishir Patil and Shiva Shankar and Shuqiang Zhang and Shuqiang Zhang and Sinong Wang and Sneha Agarwal and Soji Sajuyigbe and Soumith Chintala and Stephanie Max and Stephen Chen and Steve Kehoe and Steve Satterfield and Sudarshan Govindaprasad and Sumit Gupta and Summer Deng and Sungmin Cho and Sunny Virk and Suraj Subramanian and Sy Choudhury and Sydney Goldman and Tal Remez and Tamar Glaser and Tamara Best and Thilo Koehler and Thomas Robinson and Tianhe Li and Tianjun Zhang and Tim Matthews and Timothy Chou and Tzook Shaked and Varun Vontimitta and Victoria Ajayi and Victoria Montanez and Vijai Mohan and Vinay Satish Kumar and Vishal Mangla and Vlad Ionescu and Vlad Poenaru and Vlad Tiberiu Mihailescu and Vladimir Ivanov and Wei Li and Wenchen Wang and Wenwen Jiang and Wes Bouaziz and Will Constable and Xiaocheng Tang and Xiaojian Wu and Xiaolan Wang and Xilun Wu and Xinbo Gao and Yaniv Kleinman and Yanjun Chen and Ye Hu and Ye Jia and Ye Qi and Yenda Li and Yilin Zhang and Ying Zhang and Yossi Adi and Youngjin Nam and Yu and Wang and Yu Zhao and Yuchen Hao and Yundi Qian and Yunlu Li and Yuzi He and Zach Rait and Zachary DeVito and Zef Rosnbrick and Zhaoduo Wen and Zhenyu Yang and Zhiwei Zhao and Zhiyu Ma},
      year={2024},
      eprint={2407.21783},
      archivePrefix={arXiv},
      primaryClass={cs.AI},
      url={https://arxiv.org/abs/2407.21783}, 
}

@misc{touvron2023llama,
    title={LLaMA: Open and Efficient Foundation Language Models}, 
    author={Hugo Touvron and Thibaut Lavril and Gautier Izacard and Xavier Martinet and Marie-Anne Lachaux and Timothée Lacroix and Baptiste Rozière and Naman Goyal and Eric Hambro and Faisal Azhar and Aurelien Rodriguez and Armand Joulin and Edouard Grave and Guillaume Lample},
    year={2023},
    eprint={2302.13971},
    archivePrefix={arXiv},
    primaryClass={cs.CL},
    url={https://arxiv.org/abs/2302.13971}, 
}

@misc{touvron2023llama2,
      title={Llama 2: Open Foundation and Fine-Tuned Chat Models}, 
      author={Hugo Touvron and Louis Martin and Kevin Stone and Peter Albert and Amjad Almahairi and Yasmine Babaei and Nikolay Bashlykov and Soumya Batra and Prajjwal Bhargava and Shruti Bhosale and Dan Bikel and Lukas Blecher and Cristian Canton Ferrer and Moya Chen and Guillem Cucurull and David Esiobu and Jude Fernandes and Jeremy Fu and Wenyin Fu and Brian Fuller and Cynthia Gao and Vedanuj Goswami and Naman Goyal and Anthony Hartshorn and Saghar Hosseini and Rui Hou and Hakan Inan and Marcin Kardas and Viktor Kerkez and Madian Khabsa and Isabel Kloumann and Artem Korenev and Punit Singh Koura and Marie-Anne Lachaux and Thibaut Lavril and Jenya Lee and Diana Liskovich and Yinghai Lu and Yuning Mao and Xavier Martinet and Todor Mihaylov and Pushkar Mishra and Igor Molybog and Yixin Nie and Andrew Poulton and Jeremy Reizenstein and Rashi Rungta and Kalyan Saladi and Alan Schelten and Ruan Silva and Eric Michael Smith and Ranjan Subramanian and Xiaoqing Ellen Tan and Binh Tang and Ross Taylor and Adina Williams and Jian Xiang Kuan and Puxin Xu and Zheng Yan and Iliyan Zarov and Yuchen Zhang and Angela Fan and Melanie Kambadur and Sharan Narang and Aurelien Rodriguez and Robert Stojnic and Sergey Edunov and Thomas Scialom},
      year={2023},
      eprint={2307.09288},
      archivePrefix={arXiv},
      primaryClass={cs.CL},
      url={https://arxiv.org/abs/2307.09288}, 
}

@article{hoeffding1963prob,
     ISSN = {01621459, 1537274X},
     URL = {http://www.jstor.org/stable/2282952},
     author = {Wassily Hoeffding},
     journal = {Journal of the American Statistical Association},
     number = {301},
     pages = {13--30},
     publisher = {[American Statistical Association, Taylor & Francis, Ltd.]},
     title = {Probability Inequalities for Sums of Bounded Random Variables},
     urldate = {2026-07-27},
     volume = {58},
     year = {1963}
}

@inproceedings{ailon2006approx,
    author = {Ailon, Nir and Chazelle, Bernard},
    title = {Approximate nearest neighbors and the fast Johnson-Lindenstrauss transform},
    year = {2006},
    isbn = {1595931341},
    publisher = {Association for Computing Machinery},
    address = {New York, NY, USA},
    url = {https://doi.org/10.1145/1132516.1132597},
    doi = {10.1145/1132516.1132597},
    booktitle = {Proceedings of the Thirty-Eighth Annual ACM Symposium on Theory of Computing},
    pages = {557–563},
    numpages = {7},
    location = {Seattle, WA, USA},
    series = {STOC '06}
}

@inproceedings{
    xiang2025dfrot,
    title={{DFR}ot: Achieving Outlier-Free and Massive Activation-Free for Rotated {LLM}s with Refined Rotation},
    author={Jingyang Xiang and Sai Qian Zhang},
    booktitle={Second Conference on Language Modeling},
    year={2025},
    url={https://openreview.net/forum?id=WzGypILLDb}
}

@inproceedings{arai2025qep,
    title = {Quantization Error Propagation: Revisiting Layer-Wise Post-Training Quantization},
    author = {Arai, Yamato and Ichikawa, Yuma},
    booktitle = {Advances in Neural Information Processing Systems},
    volume = {38},
    year = {2025},
    url = {https://proceedings.neurips.cc/paper_files/paper/2025/hash/df2034a516cbd617a96492cc476276c9-Abstract-Conference.html}
}

@inproceedings{li2025gptaq,
    title = {{GPTAQ}: Efficient Finetuning-Free Quantization for Asymmetric Calibration},
    author = {Li, Yuhang and Yin, Ruokai and Lee, Donghyun and Xiao, Shiting and Panda, Priyadarshini},
    booktitle = {Proceedings of the 42nd International Conference on Machine Learning},
    pages = {36690--36706},
    year = {2025},
    volume = {267},
    series = {Proceedings of Machine Learning Research},
    publisher = {PMLR},
    url = {https://proceedings.mlr.press/v267/li25dn.html}
}

@article{jiang2023mistral,
    title = {Mistral 7B},
    author = {Jiang, Albert Q. and Sablayrolles, Alexandre and Mensch, Arthur and Bamford, Chris and Chaplot, Devendra Singh and de las Casas, Diego and Bressand, Florian and Lengyel, Gianna and Lample, Guillaume and Saulnier, Lucile and Lavaud, L{\'e}lio Renard and Lachaux, Marie-Anne and Stock, Pierre and Le Scao, Teven and Lavril, Thibaut and Wang, Thomas and Lacroix, Timoth{\'e}e and El Sayed, William},
    journal = {arXiv preprint arXiv:2310.06825},
    year = {2023},
    url = {https://arxiv.org/abs/2310.06825}
}

@misc{mistralai2024mistral7bv03,
    author = {{Mistral AI}},
    title = {Mistral 7B v0.3},
    year = {2024},
    month = may,
    howpublished = {\url{https://docs.mistral.ai/models/mistral-7b-0-3}},
    note = {Model documentation}
}

@misc{meta2024llama32,
    author = {{Meta AI}},
    title = {Llama 3.2: Revolutionizing Edge AI and Vision with Open, Customizable Models},
    year = {2024},
    howpublished = {\url{https://ai.meta.com/blog/llama-3-2-connect-2024-vision-edge-mobile-devices/}}
}
\bibliographystyle{iclr2027_conference}

\appendix
\section{An Interference View of Hadamard Rotation Bounds}
\label{app:interference_view}

In this section, we provide additional analysis supporting the randomized Hadamard rotation arguments in Section~\ref{sec:randomized_hadamard_rotation}.
We make explicit how contributions from multiple persistent CO channels can constructively interfere under a fixed Hadamard rotation and explain how random signs remove this deterministic alignment in expectation.
This interpretation complements the high-probability bounds in Theorem~\ref{thm:fixed_randomized_sampled_hadamard}.

Let $\Hmat\in\mathbb{R}^{d\times d}$ be the normalized Hadamard matrix.
For $d=2^n$, identifying the indices $j,k$ with binary vectors in $\{0,1\}^n$, we write
\begin{align}
    \Hentry{j}{k}
    =
    \frac{1}{\sqrt d}
    (-1)^{\langle j,k\rangle},
\end{align}
where $\langle j,k\rangle$ denotes the bit-wise inner product modulo $2$.

To simplify the notation, we omit channel scaling; the same argument applies by replacing $L_k$ with $L_k/\lambda_k$.
For the fixed Hadamard rotation $\Xcohad=\Xco\Hmat^\top$, each output coordinate is
\begin{align}
    \Xcohadentry{t}{j}
    =
    \frac{1}{\sqrt d}
    \sum_{k\in\Cco(\X)}
    L_k
    (-1)^{\langle j,k\rangle}.
\end{align}
Thus, although each CO channel is spread across all output coordinates, the rotated coordinate is a deterministic signed sum of the CO magnitudes.

Its squared magnitude is
\begin{align}
    E_j
    &:=
    \left|
    \Xcohadentry{t}{j}
    \right|^2
    \\
    &=
    \underbrace{
    \frac{1}{d}
    \sum_{k\in\Cco(\X)}
    L_k^2
    }_{\mathrm{individual\text{-}channel\ energy}}
    +
    \underbrace{
    \frac{1}{d}
    \sum_{\substack{k,m\in\Cco(\X)\\k\neq m}}
    L_kL_m
    (-1)^{\langle j,k\oplus m\rangle}
    }_{\mathrm{cross\text{-}channel\ interference}},
\end{align}
where $\oplus$ denotes bit-wise XOR.
The first term is independent of $j$, whereas the second term depends on the fixed Hadamard sign pattern.
When several CO contributions have the same sign at an output coordinate, the cross terms are positive and the contributions add constructively.

The individual-channel term contains $\Nco(\X)$ contributions and therefore scales linearly with $\Nco(\X)$. 
In contrast, the cross-channel term contains $\Nco(\X)(\Nco(\X)-1)$ ordered pairs.
Thus, when many of these pairwise terms align positively, the cross-channel interference can grow on the order of $\Nco(\X)^2$.
This provides an interference-based interpretation of the quadratic dependence on the number of CO channels in the fixed-rotation bound of Theorem~\ref{thm:fixed_randomized_sampled_hadamard}.

Now consider a random sign matrix
$\D=\diag(\varepsilon_1,\dots,\varepsilon_d)$,
where the $\varepsilon_k$ are independent Rademacher variables.
The corresponding coordinate energy is
\begin{align}
    E_j(\D)
    =
    \frac{1}{d}
    \left(
        \sum_{k\in\Cco(\X)}
        \varepsilon_k
        L_k
        (-1)^{\langle j,k\rangle}
    \right)^2.
\end{align}
Expanding the square gives
\begin{align}
    E_j(\D)
    =
    \frac{1}{d}
    \sum_{k\in\Cco(\X)}
    L_k^2
    +
    \frac{1}{d}
    \sum_{\substack{k,m\in\Cco(\X)\\k\neq m}}
    \varepsilon_k\varepsilon_m
    L_kL_m
    (-1)^{\langle j,k\oplus m\rangle}.
\end{align}
Since
$
\mathbb{E}_{\D}
    \left[
        \varepsilon_k\varepsilon_m
    \right]
    =
    0
    \,
    (k\neq m),
$
all cross-channel interference terms vanish in expectation:
\begin{align}
    \mathbb{E}_{\D}
    \left[
        E_j(\D)
    \right]
    =
    \frac{1}{d}
    \sum_{k\in\Cco(\X)}
    L_k^2.
\end{align}
Hence, after randomization, only the individual-channel energy remains in expectation, reducing the dependence on the number of CO channels from $\order(\Nco(\X)^2)$ to $\order(\Nco(\X))$.
This provides an interference-based interpretation of the improvement from the fixed to the randomized bound in Theorem~\ref{thm:fixed_randomized_sampled_hadamard}.

Sampling multiple sign matrices can further be interpreted as increasing the chance of obtaining a favorable sign assignment in which the CO contributions cancel more effectively, thereby reducing the largest rotated coordinate.
\section{Proofs for Section~\ref{sec:proposed_method}}
\label{app:proofs_proposed_method}

\subsection{Proof of Theorem~\ref{thm:co_regular_error_decomposition}}
\label{app:proof_co_regular_error_decomposition}

\begin{proof}
We first prove the projection-based identity.
By definition, $\Ztilde=\Z+\A$. Therefore,
\begin{align}
    \Ztilde\Vtilde^\top
    -
    \Z\V^\top
    &=
    \Ztilde\Vtilde^\top
    -
    \Ztilde\V^\top
    +
    \Ztilde\V^\top
    -
    \Z\V^\top \\
    &=
    \Ztilde
    \left(
    \Vtilde
    -
    \V
    \right)^\top
    +
    \left(
    \Ztilde
    -
    \Z
    \right)
    \V^\top \\
    &=
    \Ztilde
    \left(
    \Vtilde
    -
    \V
    \right)^\top
    +
    \A\V^\top .
\end{align}
Let $\Pimat_{\Ztilde}:=\Ztilde\Ztilde^\dagger$ denote the orthogonal projector onto the column space of $\Ztilde$, where $\Ztilde^\dagger$ is the Moore--Penrose pseudoinverse.
We decompose $\A\V^\top$ into its projection onto this column space and its orthogonal complement:
\begin{align}
    \A\V^\top
    =
    \Ztilde
    \Ztilde^\dagger
    \A\V^\top
    +
    \left(
    \I
    -
    \Pimat_{\Ztilde}
    \right)
    \A\V^\top .
\end{align}
Thus,
\begin{align}
    \Ztilde\Vtilde^\top
    -
    \Z\V^\top
    &=
    \Ztilde
    \left(
    \left(
    \Vtilde
    -
    \V
    \right)^\top
    +
    \Ztilde^\dagger
    \A\V^\top
    \right)
    +
    \left(
    \I
    -
    \Pimat_{\Ztilde}
    \right)
    \A\V^\top .
\end{align}
By the definition $(\V^\star)^\top=\V^\top-\Ztilde^\dagger\A\V^\top$, this becomes
\begin{align}
    \Ztilde\Vtilde^\top
    -
    \Z\V^\top
    =
    \Ztilde
    \left(
    \Vtilde
    -
    \V^\star
    \right)^\top
    +
    \left(
    \I
    -
    \Pimat_{\Ztilde}
    \right)
    \A\V^\top .
\end{align}
The first term belongs to the column space of $\Ztilde$, while the second term belongs to its orthogonal complement.
Hence, by orthogonality,
\begin{align}
    \obj(\Pmat,\Vtilde)
    =
    \frac{1}{T}
    \left\|
    \Ztilde
    \left(
    \Vtilde
    -
    \V^\star
    \right)^\top
    \right\|_F^2
    +
    \frac{1}{T}
    \left\|
    \left(
    \I
    -
    \Pimat_{\Ztilde}
    \right)
    \A\V^\top
    \right\|_F^2 .
\end{align}

The first squared norm is the activation-guided weight compensation term $\objagwc$ defined in Eq.~\eqref{eq:objagwc}.
The second is independent of $\Vtilde$ for fixed inputs and $\Pmat$, and is attained as the minimum local error when unrestricted weights are set to $\V^\star$.

We next upper bound this orthogonal residual using the CO and regular quantities; they are not separate terms in the exact identity.
For no-clipping dynamic per-token symmetric quantization, let
\begin{align}
    \Delta_t
    =
    \frac{
    \left\|
    \Zentry{t}{:}
    \right\|_\infty
    }{
    q_A
    },
    \qquad
    q_A
    =
    2^{b_A-1}-1.
\end{align}
For each entry,
\begin{align}
    \left|
    \Aentry{t}{j}
    \right|
    =
    \left|
    \quant_A(\Zentry{t}{j})
    -
    \Zentry{t}{j}
    \right|
    \le
    \frac{\Delta_t}{2}.
\end{align}
Therefore,
\begin{align}
    \left\|
    \A
    \right\|_F^2
    \le
    \sum_{t=1}^{T}
    \sum_{j=1}^{d}
    \frac{\Delta_t^2}{4}
    =
    \frac{d}{4q_A^2}
    \sum_{t=1}^{T}
    \left\|
    \Zentry{t}{:}
    \right\|_\infty^2 .
\end{align}
Since orthogonal projection does not increase the Frobenius norm,
\begin{align}
    \frac{1}{T}
    \left\|
    \left(
    \I
    -
    \Pimat_{\Ztilde}
    \right)
    \A\V^\top
    \right\|_F^2
    &\le
    \frac{1}{T}
    \left\|
    \A\V^\top
    \right\|_F^2 \\
    &\le
    \frac{
    \left\|
    \V
    \right\|_2^2
    }{
    T
    }
    \left\|
    \A
    \right\|_F^2 .
\end{align}
For $\Pmat=\Lambdamat^{-1}\D\Hmat^\top$, we have
$\V=\W\Pmat^{-\top}=\W\Lambdamat\D\Hmat^\top$.
Because $\D$ and $\Hmat^\top$ are orthogonal matrices,
$\|\V\|_2=\|\W\Lambdamat\|_2$.
Thus,
\begin{align}
    \frac{1}{T}
    \left\|
    \left(
    \I
    -
    \Pimat_{\Ztilde}
    \right)
    \A\V^\top
    \right\|_F^2
    \le
    \frac{
    d
    \left\|
    \W
    \Lambdamat
    \right\|_2^2
    }{
    4q_A^2 T
    }
    \sum_{t=1}^{T}
    \left\|
    \Xentry{t}{:}
    \Lambdamat^{-1}
    \D
    \Hmat^\top
    \right\|_\infty^2 .
\end{align}

Finally, decompose $\Xentry{t}{:}=\Xcoentry{t}{:}+\Xregentry{t}{:}$.
By the triangle inequality in $\ell_\infty$ and then in $\ell_2$ over tokens,
\begin{align}
    \left(
    \sum_{t=1}^{T}
    \left\|
    \Xentry{t}{:}
    \Lambdamat^{-1}
    \D
    \Hmat^\top
    \right\|_\infty^2
    \right)^{1/2}
    &\le
    \left(
    \sum_{t=1}^{T}
    \left\|
    \Xcoentry{t}{:}
    \Lambdamat^{-1}
    \D
    \Hmat^\top
    \right\|_\infty^2
    \right)^{1/2} \\
    &\quad+
    \left(
    \sum_{t=1}^{T}
    \left\|
    \Xregentry{t}{:}
    \Lambdamat^{-1}
    \D
    \Hmat^\top
    \right\|_\infty^2
    \right)^{1/2}.
\end{align}
Therefore,
\begin{align}
    \sum_{t=1}^{T}
    \left\|
    \Xentry{t}{:}
    \Lambdamat^{-1}
    \D
    \Hmat^\top
    \right\|_\infty^2
    \le
    \left(
    \sqrt{
    \objco(\Lambdamat,\D)
    }
    +
    \sqrt{
    \frac{
    \objreg(\Lambdamat,\D)
    }{
    \left\|
    \W
    \Lambdamat
    \right\|_2^2
    }
    }
    \right)^2 .
\end{align}
Combining the above inequalities gives
\begin{align}
    \obj(\Pmat,\Vtilde)
    \le
    \objagwc(\Pmat,\Vtilde)
    +
    \frac{d}{4q_A^2 T}
    \left(
    \sqrt{
    \left\|
    \W
    \Lambdamat
    \right\|_2^2
    \objco(\Lambdamat,\D)
    }
    +
    \sqrt{
    \objreg(\Lambdamat,\D)
    }
    \right)^2 .
\end{align}
\end{proof}

\subsection{Proof of Theorem~\ref{thm:fixed_randomized_sampled_hadamard}}
\label{app:proof_fixed_randomized_sampled_hadamard}

\begin{proof}
Let the normalized Hadamard matrix be written as
\begin{align}
    \entry{\Hmat^\top}{k}{j}
    =
    \frac{h_{k,j}}{\sqrt d},
    \qquad
    h_{k,j}\in\{-1,+1\}.
\end{align}

\paragraph{Fixed rotation.}
Under the persistent sparse CO model,
\begin{align}
    \left(
        \Xcoentry{t}{:}
        \Lambdamat^{-1}
        \Hmat^\top
    \right)_j
    =
    \frac{1}{\sqrt d}
    \sum_{k\in\Cco(\X)}
    \frac{L_k}{\lambda_k}
    h_{k,j}.
\end{align}
Therefore,
\begin{align}
    \left|
        \left(
            \Xcoentry{t}{:}
            \Lambdamat^{-1}
            \Hmat^\top
        \right)_j
    \right|
    \le
    \frac{1}{\sqrt d}
    \sum_{k\in\Cco(\X)}
    \frac{|L_k|}{\lambda_k}
    \le
    \frac{\Nco(\X)}{\sqrt d}
    L_{\Lambda,\max}.
\end{align}
Taking the maximum over $j$ and summing over $T$ tokens gives
\begin{align}
    \objco(\Lambdamat,\mathbf{I})
    \le
    \frac{T\Nco(\X)^2}{d}
    L_{\Lambda,\max}^2.
\end{align}

\paragraph{Randomized rotation.}
Consider a random sign matrix $\D=\operatorname{diag}(\varepsilon_1,\ldots,\varepsilon_d)$, where $\varepsilon_k$ are independent Rademacher variables.
For a fixed coordinate $j$,
\begin{align}
    Y_j
    =
    \left(
        \Xcoentry{t}{:}
        \Lambdamat^{-1}
        \D
        \Hmat^\top
    \right)_j
    =
    \frac{1}{\sqrt d}
    \sum_{k\in\Cco(\X)}
    \frac{L_k}{\lambda_k}
    \varepsilon_k
    h_{kj}.
\end{align}
This is a centered Rademacher sum.
By Hoeffding's inequality~\citep{hoeffding1963prob},
\begin{align}
    \Pr \left(|Y_j|>u \right)
    \le
    2 \exp
    \left(
        - \frac{d u^2}{2R_\Lambda^2}
    \right),
\end{align}
where $R_\Lambda^2=\sum_{k\in\Cco(\X)}L_k^2/\lambda_k^2$.
Taking a union bound over $d$ coordinates gives
\begin{align}
    \Pr
    \left(
        \left\|
            \Xcoentry{t}{:}
            \Lambdamat^{-1}
            \D
            \Hmat^\top
        \right\|_\infty
        > u
    \right)
    \le
    2d \exp
    \left(
        - \frac{d u^2}{2R_\Lambda^2}
    \right).
\end{align}
Choosing
\begin{align}
    u^2
    =
    \frac{2R_\Lambda^2}{d}
    \left(
        \log \left( 2d \right)
        +
        \log
        \frac{1}{\delta}
    \right)
\end{align}
gives, with probability at least $1-\delta$,
\begin{align}
    \left\|
        \Xcoentry{t}{:}
        \Lambdamat^{-1}
        \D
        \Hmat^\top
    \right\|_\infty^2
    \le
    \frac{2R_\Lambda^2}{d}
    \left(
        \log \left( 2d \right)
        +
        \log
        \frac{1}{\delta}
    \right)
\end{align}
Because the persistent sparse CO vector is shared across tokens, summing over $T$ tokens yields
\begin{align}
    \objco(\Lambdamat,\D)
    \le
    \frac{2T}{d}
    R_\Lambda^2
    \left(
        \log \left( 2d \right)
        +
        \log
        \frac{1}{\delta}
    \right)
\end{align}
Since $R_\Lambda^2\le\Nco(\X)L_{\Lambda,\max}^2$, the stated $\Nco(\X)$-dependent bound follows.

\paragraph{Sampled rotation.}
Consider $N_s$ independently sampled sign matrices $\mathcal{D}_{N_s}=\{\D^{(1)},\ldots,\D^{(N_s)}\}$.
For a single sample,
\begin{align}
\Pr
\left(
\objco(\Lambdamat,\D)
>
T u^2
\right)
\le
2d
\exp
\left(
-
\frac{
d u^2
}{
2R_\Lambda^2
}
\right).
\end{align}
Hence,
\begin{align}
\Pr
\left(
\min_{\D\in\mathcal{D}_{N_s}}
\objco(\Lambdamat,\D)
>
T u^2
\right)
\le
\left[
2d
\exp
\left(
-
\frac{
d u^2
}{
2R_\Lambda^2
}
\right)
\right]^{N_s}.
\end{align}
Setting the right-hand side to $\delta$ gives
\begin{align}
u^2
=
\frac{
2R_\Lambda^2
}{
d
}
\left(
\log(2d)
+
\frac{1}{N_s}
\log\frac{1}{\delta}
\right).
\end{align}
Therefore, with probability at least $1-\delta$,
\begin{align}
\min_{\D\in\mathcal{D}_{N_s}}
\objco(\Lambdamat,\D)
\le
\frac{
2T
}{
d
}
R_\Lambda^2
\left(
\log(2d)
+
\frac{1}{N_s}
\log\frac{1}{\delta}
\right).
\end{align}
Using again $R_\Lambda^2\le\Nco(\X)L_{\Lambda,\max}^2$ gives the stated bound.
\end{proof}

\subsection{Proof of Proposition~\ref{prop:l2_scaling_regular_residual}}
\label{app:proof_l2_scaling_regular_residual}

\begin{proof}
Let $\boldsymbol{r}_{t}=\Xregentry{t}{:}\Lambdamat^{-1}$.
For a fixed token $t$ and coordinate $j$,
\begin{align}
    Y_{t,j}
    =
    \left(
    \boldsymbol{r}_{t}
    \D
    \Hmat^\top
    \right)_j
    =
    \frac{1}{\sqrt d}
    \sum_{k=1}^{d}
    r_{t,k}
    \varepsilon_k
    h_{k,j}.
\end{align}
This is a centered Rademacher sum.
By Hoeffding's inequality,
\begin{align}
    \Pr
    \left(
    |Y_{t,j}|>u
    \right)
    \le
    2
    \exp
    \left(
    -
    \frac{
    d u^2
    }{
    2
    \left\|
    \boldsymbol{r}_{t}
    \right\|_2^2
    }
    \right).
\end{align}
Taking a union bound over all $T$ tokens and $d$ coordinates, with probability at least $1-\delta$,
\begin{align}
    \left\|
    \boldsymbol{r}_{t}
    \D
    \Hmat^\top
    \right\|_\infty^2
    \le
    \frac{
    2\log(2dT/\delta)
    }{
    d
    }
    \left\|
    \boldsymbol{r}_{t}
    \right\|_2^2
\end{align}
simultaneously for all $t$.
Summing over $t$ gives
\begin{align}
    \sum_{t=1}^{T}
    \left\|
    \Xregentry{t}{:}
    \Lambdamat^{-1}
    \D
    \Hmat^\top
    \right\|_\infty^2
    &\le
    \frac{
    2\log(2dT/\delta)
    }{
    d
    }
    \sum_{t=1}^{T}
    \left\|
    \Xregentry{t}{:}
    \Lambdamat^{-1}
    \right\|_2^2 \\
    &=
    \frac{
    2\log(2dT/\delta)
    }{
    d
    }
    \left\|
    \Xreg
    \Lambdamat^{-1}
    \right\|_F^2.
\end{align}
Multiplying both sides by $\left\|\W\Lambdamat\right\|_2^2$ yields
\begin{align}
    \objreg(\Lambdamat,\D)
    \le
    \frac{2\log(2dT/\delta)}{d}
    \left\|
        \W
        \Lambdamat
    \right\|_2^2
    \left\|
        \Xreg
        \Lambdamat^{-1}
    \right\|_F^2 .
\end{align}
Thus, the regular quantity $\objreg$ is controlled by
$\|\W\Lambdamat\|_2^2\|\Xreg\Lambdamat^{-1}\|_F^2$
up to constants and logarithmic factors.

We derive a channel-wise scaling rule using the Frobenius surrogate
\begin{align}
    \mathcal{R}_2(\Lambdamat)
    =
    \left\|
    \W
    \Lambdamat
    \right\|_F^2
    \left\|
    \Xreg
    \Lambdamat^{-1}
    \right\|_F^2 .
\end{align}
Let $a_k=\|\Xregentry{:}{k}\|_2^2$ and $b_k=\|\Wentry{:}{k}\|_2^2$.
Then,
\begin{align}
    \mathcal{R}_2(\Lambdamat)
    =
    \left(
    \sum_{k=1}^{d}
    \lambda_k^2 b_k
    \right)
    \left(
    \sum_{k=1}^{d}
    \frac{a_k}{\lambda_k^2}
    \right).
\end{align}
By Cauchy--Schwarz,
\begin{align}
    \left(
    \sum_{k=1}^{d}
    \lambda_k^2 b_k
    \right)
    \left(
    \sum_{k=1}^{d}
    \frac{a_k}{\lambda_k^2}
    \right)
    \ge
    \left(
    \sum_{k=1}^{d}
    \sqrt{a_k b_k}
    \right)^2.
\end{align}
Equality holds when $\lambda_k^2b_k\propto a_k/\lambda_k^2$, or equivalently,
$\lambda_k^4\propto a_k/b_k$.
Thus, up to a common positive scalar,
\begin{align}
    \lambda_k
    =
    \left(
    \frac{a_k}{b_k}
    \right)^{1/4}
    =
    \sqrt{
    \frac{
    \left\|
    \Xregentry{:}{k}
    \right\|_2
    }{
    \left\|
    \Wentry{:}{k}
    \right\|_2
    }
    }.
\end{align}
This proves the stated $L_2$ scaling rule.
\end{proof}

\section{Full Results Across Models}
\label{app:full_results}

Table~\ref{tb:evaluation_full} reports C4 perplexity and the individual zero-shot task accuracies omitted from Table~\ref{tb:evaluation}.
All settings and baselines are identical to those in Section~\ref{sec:experiments}.

\begin{table}[p]
\caption{
\label{tb:evaluation_full} Full W4A4 results with GPTAQ weight quantization and KV4 across model families and scales.
}
\centering
\renewcommand{\arraystretch}{1.1}
\resizebox{\textwidth}{!}{
\begin{tabular}{c|c c|c c|c c c c c c c} \hline
\multirow{2}{*}{Model} & \multirow{2}{*}{Bits} & \multirow{2}{*}{Method} & \multicolumn{2}{c|}{PPL $\downarrow$} & \multicolumn{7}{c}{Accuracy (\%) $\uparrow$} \\ \cline{4-12}
& & & WT2 & C4 & PIQA & ARC-E & ARC-C & HS & WG & LB & Avg. \\ \hline
\multirow{5}{*}{Llama-7B} & W16A16 & N/A & 5.67 & 8.03 & 79.16 & 72.85 & 44.62 & 76.19 & 70.01 & 73.51 & 69.39 \\ \cdashline{2-12}[2pt/1pt]
& \multirow{4}{*}{W4A4} & \quarot & 6.12 & 8.84 & 77.53 & 68.01 & 41.38 & 72.82 & 66.77 & 70.54 & 66.17 \\
& & \smoothrot & 6.13 & 8.87 & 77.58 & 67.72 & 40.27 & 72.66 & 67.88 & 70.83 & 66.16 \\
& & \spinquant & 6.19 & 8.95 & 76.44 & 69.07 & 42.32 & 72.75 & 67.01 & 69.51 & 66.18 \\
& & \proposal & \textbf{6.09} & \textbf{8.76} & 77.15 & 69.65 & 41.30 & 73.01 & 67.64 & 71.28 & \textbf{66.67} \\ \hline
\multirow{5}{*}{Llama-13B} & W16A16 & N/A & 5.05 & 7.46 & 80.20 & 74.75 & 47.87 & 79.08 & 72.61 & 76.19 & 71.78 \\ \cdashline{2-12}[2pt/1pt]
& \multirow{4}{*}{W4A4} & \quarot & 5.42 & 8.00 & 78.51 & 72.01 & 46.16 & 76.66 & 69.85 & 73.78 & 69.49 \\
& & \smoothrot & 5.42 & 8.03 & 78.62 & 70.58 & 43.94 & 76.67 & 70.88 & 73.86 & 69.09 \\
& & \spinquant & 5.44 & 8.13 & 78.62 & 71.76 & 46.33 & 76.89 & 69.14 & 73.12 & 69.31 \\
& & \proposal & \textbf{5.37} & \textbf{7.96} & 78.29 & 72.18 & 47.01 & 76.68 & 70.72 & 74.60 & \textbf{69.91} \\ \hline
\multirow{5}{*}{Llama2-7B} & W16A16 & N/A & 5.47 & 7.92 & 79.11 & 74.58 & 46.25 & 76.00 & 69.06 & 73.90 & 69.82 \\ \cdashline{2-12}[2pt/1pt]
& \multirow{4}{*}{W4A4} & \quarot & 5.96 & 8.83 & 76.55 & 70.54 & 41.98 & 72.78 & 65.43 & 70.66 & 66.32 \\
& & \smoothrot & 5.98 & 8.87 & 76.39 & 71.00 & 41.72 & 72.89 & 66.14 & 70.56 & 66.45 \\
& & \spinquant & 6.06 & 9.03 & 76.82 & 69.32 & 41.98 & 72.28 & 64.80 & 70.70 & 65.98 \\
& & \proposal & \textbf{5.94} & \textbf{8.80} & 77.80 & 71.09 & 41.04 & 72.31 & 66.77 & 70.52 & \textbf{66.59} \\ \hline
\multirow{5}{*}{Llama2-13B} & W16A16 & N/A & 4.86 & 7.36 & 80.52 & 77.53 & 49.06 & 79.38 & 72.14 & 76.77 & 72.57 \\ \cdashline{2-12}[2pt/1pt]
& \multirow{4}{*}{W4A4} & \quarot & 5.24 & 8.01 & 78.73 & 73.44 & 45.14 & 76.31 & 70.72 & 74.89 & 69.87 \\
& & \smoothrot & 5.26 & 8.04 & 78.40 & 75.51 & 46.93 & 76.25 & 69.14 & 74.67 & 70.15 \\
& & \spinquant & 5.26 & 8.14 & 77.80 & 75.38 & 46.84 & 76.79 & 69.77 & 74.81 & 70.23 \\
& & \proposal & \textbf{5.23} & \textbf{7.99} & 78.51 & 74.54 & 47.70 & 76.93 & 70.64 & 74.69 & \textbf{70.50} \\ \hline
\multirow{5}{*}{Llama3-8B} & W16A16 & N/A & 5.94 & 8.80 & 80.85 & 77.74 & 53.41 & 79.16 & 72.61 & 76.05 & 73.30 \\ \cdashline{2-12}[2pt/1pt]
& \multirow{4}{*}{W4A4} & \quarot & 7.45 & 12.16 & 76.66 & 68.77 & 41.72 & 73.51 & 65.59 & 67.98 & 65.71 \\
& & \smoothrot & 7.57 & 12.45 & 75.30 & 67.05 & 43.00 & 73.34 & 69.06 & 67.38 & 65.85 \\
& & \spinquant & 7.33 & 12.04 & 76.93 & 74.92 & 46.50 & 75.45 & 67.88 & 68.50 & \textbf{68.36} \\
& & \proposal & \textbf{7.20} & \textbf{11.71} & 77.09 & 73.32 & 45.56 & 74.85 & 68.35 & 69.45 & 68.10 \\ \hline
\multirow{5}{*}{Llama3.2-1B} & W16A16 & N/A & 9.41 & 13.05 & 74.54 & 60.61 & 36.35 & 63.68 & 60.69 & 62.95 & 59.80 \\ \cdashline{2-12}[2pt/1pt]
& \multirow{4}{*}{W4A4} & \quarot & 13.81 & 21.52 & 67.74 & 50.88 & 30.20 & 53.19 & 53.91 & 41.57 & 49.58 \\
& & \smoothrot & 14.13 & 22.05 & 67.19 & 50.76 & 30.38 & 52.42 & 53.51 & 37.51 & 48.63 \\
& & \spinquant & 13.33 & 20.46 & 66.87 & 51.47 & 30.80 & 55.19 & 52.88 & 45.02 & 50.37 \\
& & \proposal & \textbf{12.92} & \textbf{19.87} & 68.55 & 51.81 & 31.83 & 54.17 & 53.04 & 44.83 & \textbf{50.71} \\ \hline
\multirow{5}{*}{Llama3.2-3B} & W16A16 & N/A & 7.55 & 10.46 & 77.48 & 71.72 & 45.99 & 73.58 & 70.01 & 70.44 & 68.20 \\ \cdashline{2-12}[2pt/1pt]
& \multirow{4}{*}{W4A4} & \quarot & 9.27 & 14.29 & 74.05 & 61.99 & 38.14 & 67.53 & 63.06 & 61.93 & 61.12 \\
& & \smoothrot & 9.35 & 14.52 & 72.80 & 60.19 & 38.31 & 66.42 & 61.09 & 60.66 & 59.91 \\
& & \spinquant & 9.19 & 14.36 & 73.72 & 63.93 & 41.21 & 68.22 & 63.77 & 61.32 & \textbf{62.03} \\
& & \proposal & \textbf{9.08} & \textbf{14.04} & 72.47 & 58.54 & 39.33 & 67.94 & 64.40 & 62.97 & 60.94 \\ \hline
\multirow{5}{*}{Mistral-7B-v0.3} & W16A16 & N/A & 5.35 & 8.28 & 82.26 & 78.24 & 52.22 & 80.47 & 73.88 & 75.33 & 73.73 \\ \cdashline{2-12}[2pt/1pt]
& \multirow{4}{*}{W4A4} & \quarot & 5.73 & 8.94 & 80.14 & 75.67 & 47.70 & 77.70 & 69.46 & 72.60 & 70.55 \\
& & \smoothrot & 5.77 & 8.99 & 80.41 & 73.95 & 45.56 & 77.40 & 68.75 & 72.70 & 69.80 \\
& & \spinquant & 5.74 & 8.95 & 80.41 & 75.88 & 49.15 & 77.96 & 70.01 & 73.61 & \textbf{71.17} \\
& & \proposal & \textbf{5.71} & \textbf{8.87} & 79.49 & 75.76 & 48.89 & 78.04 & 70.09 & 73.76 & 71.00 \\ \hline
\end{tabular}
}
\end{table}

\clearpage

\end{document}